\documentclass[11pt]{article}
\usepackage{latexsym,amsfonts,amsmath,amsthm,amssymb, epsfig}
\usepackage{color}
\usepackage[normalem]{ulem}
\usepackage{hyperref}
\usepackage{todonotes}
\usepackage{caption}

\usepackage{trimclip}
\newlength{\trianglewidth}
\newlength{\pluswidth}
\newcommand{\righttriangleplus}{%
    \mathrel{\makebox[\trianglewidth]{%
        \raisebox{.3\height}{\clipbox{.3\width{} .3\height}{\(+\)}}%
        \hspace*{-.333\pluswidth}%
        \makebox[\trianglewidth]{\(>\)}%
    }}%
}

\newcommand{\ReLU}{\mathrm{ReLU}}

\usepackage[vcentering,dvips]{geometry}
\newtheorem{thm}{Theorem}
\newtheorem{lem}[thm]{Lemma}
\newtheorem{prop}[thm]{Proposition}
\newtheorem{cor}[thm]{Corollary}

\theoremstyle{remark}
\newtheorem{rem}{Remark}

\theoremstyle{definition}
\newtheorem{dfn}[thm]{Definition}

\newcommand{\real}{\mathbb{R}}
\newcommand{\R}{\mathbb{R}}

\newcommand{\N}{\mathbb{N}}
\newcommand{\C}{\mathbb{C}}
\newcommand{\CalC}{\mathcal{C}}
\newcommand{\CalS}{\mathcal{S}}
\newcommand{\cR}{\mathcal{R}}
\newcommand{\cN}{\mathcal{N}}
\newcommand{\Z}{\mathbb{Z}}
\newcommand{\F}{\mathcal F}
\newcommand{\eps}{\varepsilon}

\newcommand{\vol}{\operatornamewithlimits{vol}}

\newcommand{\BB}{\mathbb{B}}
\newcommand{\CalB}{\mathcal{B}}

\newcommand{\be}{\begin{equation}}
\newcommand{\ee}{\end{equation}}

\newlength{\fixboxwidth}
\begin{document}

\title{Nonlocal techniques for the analysis of \\
deep ReLU neural network approximations
}
\author{C. Schneider\footnote{Friedrich-Alexander Universit\"at Erlangen, Applied Mathematics III, Cauerstr. 11, 91058 Erlangen, Germany. Email: \href{mailto:schneider@math.fau.de}{schneider@math.fau.de}}, \,
M. Ullrich\footnote{Institute of Analysis \& Department of Quantum Information and Computation, Johannes Kepler University 
Linz, Austria. Email: \href{mailto:mario.ullrich@jku.at}{mario.ullrich@jku.at}. 
MU is supported by the Austrian Federal Ministry of Education, Science and Research via the Austrian Research Promotion Agency (FFG) through the project FO999921407 (HDcode) funded by the European Union via NextGenerationEU.}, \,
and J. Vyb\'\i ral\footnote{Department of Mathematics, Faculty of Nuclear Sciences and Physical Engineering,
Czech Technical University in Prague, Trojanova 13, 12000 Praha, Czech Republic. Email: \href{mailto:jan.vybiral@fjfi.cvut.cz}{jan.vybiral@fjfi.cvut.cz}.
The work of this author has been supported by the grant P202/23/04720S of the Grant Agency of the Czech Republic.
J.V. is a member of the Nečas center for mathematical modeling.}}

\maketitle

\begin{abstract}
Recently, Daubechies, DeVore, Foucart, Hanin, and Petrova introduced a system of piece-wise linear
functions, which can be easily reproduced by artificial neural networks with the ReLU activation function and which form a Riesz basis of $L_2([0,1])$.
This work was generalized by two of the authors to the multivariate setting. We show that this system serves as a Riesz basis
also for Sobolev spaces $W^s([0,1]^d)$ and Barron classes ${\mathbb B}^s([0,1]^d)$ with smoothness $0<s<1$.
We apply this fact to re-prove some recent results on the approximation of functions from these classes by deep neural networks.
Our proof method avoids using local approximations and allows us to track also the implicit constants as well as  to show that 
we can avoid the curse of dimension. 
Moreover, we also study how well one can approximate Sobolev and Barron functions by ANNs if only function values are known.

\medskip

{\em MSC2020 Math Subject Classifications:} 68T07, 42C15, 11A25.
\end{abstract}


\section{Introduction}

The authors of \cite{DD+22} introduced a system of piece-wise linear functions,
which resembles the behavior of  a trigonometric orthonormal basis on the interval $[0,1]$
but, simultaneously, can be also easily reproduced by artificial neural networks with the $\ReLU(x):=\max\{x,0\}$ activation function.
To be more precise, for $x\in[0,1]$, we consider  
\begin{equation}\label{C-intro}
{\CalC} (x):=4\left|x-\frac{1}{2}\right|-1
=\begin{cases}1-4x,\ x\in[0,1/2),\\4x-3,\ x\in[1/2,1]\end{cases}
\end{equation}
and
\[
{\CalS}(x):=\left|2-4\left|x-\frac{1}{4}\right|\right|-1=\begin{cases}
4x,\ &x\in[0,1/4),\\
2-4x,\ &x\in[1/4,3/4),\\
4x-4, &x\in [3/4,1].
\end{cases}
\]
Indeed, $\CalC(x)$ and $\CalS(x)$ are piece-wise linear functions, which interpolate $\cos(2\pi x)$ and $\sin(2\pi x)$ for $x\in\{0,1/4,1/2,3/4,1\}$.
For $x\in\R$, we extend this definition periodically, i.e. $\CalC(x)=\CalC(x-\lfloor x\rfloor)$ and $\CalS(x)=\CalS(x-\lfloor x\rfloor)$. 
Moreover, for $k\ge 1$ and $x\in\R$, we put $\CalC_k(x):=\CalC(kx)$ and $\CalS_k(x):=\CalS(kx)$.

It was shown in \cite{DD+22}, that the set
$\{\CalC_k,\,\CalS_k\colon k\in\N\}$
forms a Riesz basis of the subspace of $L_2([0,1])$ of functions with vanishing mean. This means that
\begin{equation}\label{eq:Riesz}
c(\|\alpha\|_2^2+\|\beta\|_2^2)\le \left\|\sum_{k=1}^\infty(\alpha_k \CalC_k+\beta_k\CalS_k)\right\|^2_2\le C(\|\alpha\|_2^2+\|\beta\|_2^2)
\end{equation}
holds for two absolute constants $C,c>0$ and for any two real sequences $\alpha=(\alpha_k)_{k=1}^\infty$ and $\beta=(\beta_k)_{k=1}^\infty$
and that every function from $L_2([0,1])$ with vanishing mean lies in the closed linear span of it. 
The values of $C,c>0$ can be chosen, for example, as $c=1/6$ and $C=1/2$. 
The functions $\CalC_k$ and $\CalS_k$ 
have $L_2$-norm 
$3^{-1/2}$, but this is not important in what follows. 
Note that \eqref{eq:Riesz} would hold for the normalized system with constants $1/2$ and $3/2$ instead.
It is also easy to see that adding a constant function 1 to ${\mathcal R}_1$
makes it a Riesz basis of the whole $L_2([0,1])$. Finally, \cite[Theorem 6.2]{DD+22} provides, for every $k\ge 1$, a construction of a feed-forward artificial neural network (abbreviated by ANN in the sequel)
with $L=\lceil \log_2k\rceil+1$ hidden layers and with 2 artificial neurons in each layer, which reconstructs $\CalC_k$ on $[0,1]$.
The same is true for $\CalS_k$ if we increase the number of layers by one.

\medskip

The generalization of this approach to functions of $d\ge 2$ variables and the space $L_2([0,1]^d)$ is not straightforward, and the seemingly simple way
of taking tensor products 
suffers a number of drawbacks.
Firstly, such functions could not be exactly recovered by ANNs with the ReLU activation function, they could only
be approximated to some limited precision, cf. \cite{EPGB,T,Y17}. And secondly, the ratio of the optimal constants $C$ and $c$ in the corresponding version of \eqref{eq:Riesz}
would grow exponentially with the underlying dimension $d$.

A surprisingly simple and effective generalization of 
the univariate Riesz basis
to higher dimensions was discovered in \cite{SV}. 
For this, let us define
\[
\CalC_k(x)\,:=\,\CalC(k\cdot x)\qquad\text{and}\qquad \CalS_k(x)\,:=\,\CalS(k\cdot x), 
\]
where $k\cdot x=\sum_{j}k_j x_j$ is the usual inner product of $k\in\Z^d$ and $x\in\R^d$, 
and define the system 

\begin{equation}\label{eq:intro_Riesz_d}
{\mathcal R}_d:=\{1\}\cup\Bigl\{
\CalC_k,\,\CalS_k
\colon k\in\Z^d,\, k\,{\righttriangleplus}\, 0 \Bigr\}.
\end{equation}

Here, $k\,\righttriangleplus\, 0$
means that $k=(k_1,\dots,k_d)\in\Z^d$ is not equal to zero and the first non-zero entry of $k$ is positive.
It was shown in \cite{SV} that ${\mathcal R}_d$ is a Riesz basis of $L_2([0,1]^d)$ for every $d\ge 2$
and that the constants $C,c>0$ in the corresponding analogue of \eqref{eq:Riesz} can be chosen again as 
$c=1/6$ and $C=1/2$ (if we leave out the constant function from ${\mathcal R}_d$, otherwise $C=1$),  independently of $d$.
Again, the elements of \eqref{eq:intro_Riesz_d} have $L_2$-norm~$3^{-1/2}$, and 
can easily be  reproduced by ANNs with the ReLU activation function, see Section~\ref{sec:ANN}.

\medskip

The first aim of the present paper is to extend the results of \cite{SV} to other function spaces than $L_2([0,1]^d)$.
In fact, we show that properly normalized analogues of ${\mathcal R}_d$ are Riesz bases of the Sobolev spaces $W^s([0,1]^d)$ 
and the Barron classes ${\mathbb B}^s([0,1]^d)$ 
for every $0<s<1$, see Section~\ref{Sec:2} for the exact statements.
In both cases, the constants in \eqref{eq:Riesz} can again be chosen independently of~$d$. 

\medskip

Since the appearance of ANNs, many authors investigated, which functions can be computed (or well-approximated) by ANNs of a given structure.
This research field is known as \emph{expressivity} of ANNs \cite{Ragu} and it aims to help to explain the empirical success of ANNs.
An important breakthrough was achieved already in \cite{Barron}, where a general result about convex hulls in Hilbert spaces (attributed to Maurey in \cite{Pisier})
was exploited to show that approximation of functions with finite first Fourier moment by ANNs does not 
suffer the curse of dimensionality, in contrast to any linear method of approximation. 
These function classes are nowadays called Barron classes (together with their numerous variants).
Since then, the approximation of functions from different function spaces (Sobolev spaces, functions of bounded variation,
Lipschitz continuous functions, etc.) using neural networks attracted a lot of attention 
and many optimal and nearly optimal results are nowadays available.
We refer to~\cite{EMW22,EW-Banach,E1} for more on the representation of ANNs, 
as well as~\cite{Barron2,BGKP,DD+22,DeVore-2, Mhaskar', Mhaskar, 50, S23, Y17} and references therein for further information on approximation with ANNs.

\medskip

In Section~\ref{sec:ANN} we apply our new  results concerning the behavior of \eqref{eq:intro_Riesz_d} in Sobolev and Barron spaces
to the approximation of functions from these spaces by deep ANNs. We avoid any use of local approximation, which usually leads to a bad dependence of the implicit constants on the dimension and to some quite technical computations.
Instead, we work exclusively with the building blocks of \eqref{eq:intro_Riesz_d}, which are defined on the whole unit cube of $\R^d$. The proof method we use is actually quite straightforward - we decompose a given $f$
from one of these function spaces into a series involving the basis functions from \eqref{eq:intro_Riesz_d}. Truncating this series at a suitable
position splits the series into two parts. The first one is recovered exactly by a suitably chosen ANN, the second one is simply the error of approximation.
For Barron spaces, we need to combine this approach with the concept of \emph{best $n$-term approximation}, well-known in non-linear approximation theory,
to select the most important terms from the series decomposition of $f$. Our technique allows us to work exclusively on the sequence space level and,
by the properties of \eqref{eq:intro_Riesz_d} as shown in Section \ref{Sec:2}, we virtually pay  no price for this step. 
In particular,
we are able to track the dependence of the parameters on the underlying dimension $d$ and we show that in some cases we can indeed avoid their exponential
dependence on $d$. 
In this way, we reprove some of the known results but using an essentially different technique.


\medskip

Finally, in Section~\ref{sec:discussion}, 
we discuss our results and compare them to the existing literature. 
We also comment on learning the specific ANN 
based on (randomly chosen) function values of $f$ in Section \ref{sec:sampling}.

\section{Riesz bases of Sobolev and Barron classes}\label{Sec:2}

The aim of this section is to study the properties of the univariate system 
$\cR_1$ and its multivariate analogue $\cR_d$ from \eqref{eq:intro_Riesz_d}
in other function spaces than just $L_2([0,1])$ or $L_2([0,1]^d)$, respectively. To simplify the presentation,
we first deal with the univariate Sobolev spaces before we come their high-dimensional counterparts and to the Barron classes.

\subsection{Univariate Sobolev classes}

Let us consider a real-valued square-integrable function $f:{\mathbb R}\to {\mathbb R}$, which is periodic with period one, i.e., $f$ is given by
\begin{equation}\label{eq:decomp_f_Fourier}
f(x)\,=\, a_0 \,+\, \sum_{m=1}^\infty a_m\cos(2\pi mx) + b_m\sin(2\pi mx),\quad x\in [0,1].
\end{equation}
For a real number $s\ge 0$, we define the usual Sobolev spaces of periodic functions of order~$s$ by 
\begin{align*}
W^s([0,1])&:=\left\{f:f \text{ is given by \eqref{eq:decomp_f_Fourier} and }\|f\|_{W^s}^2:= 
a_0^2+
\sum_{m=1}^\infty m^{2s}(a_m^2+b_m^2)<\infty\right\},
\end{align*}
and note that $W^0([0,1])=L_2([0,1])$.

It was shown in \cite{DD+22} that the system $\cR_1$ from \eqref{eq:intro_Riesz_d} 
forms a Riesz basis in $W^0([0,1])$. The aim of this section is to investigate the properties of 
$\cR_1$ 
as subsets of $W^s([0,1])$.
We start by defining the analogues of $W^s([0,1])$ based on \eqref{eq:intro_Riesz_d}.
Let
\begin{equation}\label{eq:decomp_f_CalC}
f(x) = \alpha_0 \,+\, \sum_{k=1}^\infty \alpha_k \CalC_k(x) + \beta_k \CalS_k(x), \quad x\in [0,1],
\end{equation}
with $\alpha_k,\beta_k\in\R$, 
which converges in $L_2([0,1])$ if $\alpha=(\alpha_k)_{k=1}^\infty$ and $\beta=(\beta_k)_{k=1}^\infty$ are square-summable.
For $s\ge 0$, we define
\begin{align*}
\F^s([0,1])\,:=\,\left\{f:f \text{ is given by \eqref{eq:decomp_f_CalC} and }\|f\|_{\F^s}^2:=\alpha_0^2 \,+\,\sum_{m=1}^\infty m^{2s}(\alpha_m^2+\beta_m^2)<\infty\right\}.
\end{align*}

The following theorem shows that $\F^s$ is  a useful tool  for the analysis of $W^s$.

\begin{thm}\label{thm:d1_WF} Let $0\le s<1$. Then, $W^s([0,1])=\F^s([0,1])$ in the sense of equivalent norms. 
Moreover, the system
\begin{equation}\label{eq:Riesz_1}
    \{1\}\cup\left\{k^{-s}\CalC_k, k^{-s}\CalS_k\colon k\in \N\right\}
\end{equation}
is 
a Riesz basis of $W^s([0,1])$ 
and for constants $c,C>0$ it holds  
\begin{equation}\label{ineq-Riesz-1d}
c\,\sum_{k\in \mathbb{N}}
(\alpha_k^{2}+\beta_k^2) 
\;\le\; 
\left\|\sum_{k\in \mathbb{N}}k^{-s}\Bigl(\alpha_k\CalC_k+ 
\beta_k \CalS_k\Bigr)\right\|^2_{W^s}
\;\le\; C\,\sum_{k\in \mathbb{N}}
(\alpha_k^{2}+\beta_k^2). 
\end{equation}
Modifications of \eqref{ineq-Riesz-1d} for functions with constant term are obvious (by incorporating $\alpha_0$).
\end{thm}

\begin{proof} 
We assume without any loss of generality that $a_0=\alpha_0=0$ throughout the proof.

\emph{Step 1.} First, we show that $W^s([0,1])\hookrightarrow \F^s([0,1])$.
Let $f$ be given by \eqref{eq:decomp_f_Fourier}.
We combine it with \cite{DD+22} (cf. also \cite[Lemma 2.5]{SV}), which states that
\begin{equation}\label{eq:cos_Riesz}
\sqrt{2}\cos(2\pi x)=\sum_{\ell=0}^\infty \frac{\mu(2\ell+1)}{(2\ell+1)^2}\cdot \frac{\sqrt{3}}{\kappa}{\mathcal C}_{2\ell+1}(x),
\end{equation}
\begin{equation}\label{eq:sin_Riesz}
\sqrt{2}\sin(2\pi x)=\sum_{\ell=0}^\infty (-1)^\ell\frac{\mu(2\ell+1)}{(2\ell+1)^2}\cdot \frac{\sqrt{3}}{\kappa}{\mathcal S}_{2\ell+1}(x), 
\end{equation}
where $\kappa^2=96/\pi^4$ and 
$\mu(n)\in\{-1,0,1\}$ denotes the  M\"obius function, i.e., the sum of the primitive $n$-th roots of unity.
Let us note, that $\mu(n)=0$ if the $n$ is not square-free, i.e., if it is divisible by some squared prime. 

Plugging \eqref{eq:cos_Riesz} and \eqref{eq:sin_Riesz} into \eqref{eq:decomp_f_Fourier}, we obtain
\begin{align*}
f(x)&=\frac{\sqrt{3}}{\sqrt{2}\,\kappa}\biggl\{\sum_{m=1}^\infty a_m \sum_{\ell=0}^\infty \frac{\mu(2\ell+1)}{(2\ell+1)^2}\, {\mathcal C}_{(2\ell+1)m}(x)\\
&\qquad \qquad +\sum_{m=1}^\infty b_m \sum_{\ell=0}^\infty (-1)^\ell\frac{\mu(2\ell+1)}{(2\ell+1)^2}\, {\mathcal S}_{(2\ell+1)m}(x)\biggr\}\\
&=\frac{\sqrt{3}}{\sqrt{2}\,\kappa}\sum_{k=1}^\infty \Bigl[\alpha_k \CalC_k(x) + \beta_k \CalS_k(x)\Bigr],
\end{align*}
where
\begin{gather*}
\alpha_k=\sum_{(\ell,m):k=(2\ell+1)m} a_m\cdot\frac{\mu(2\ell+1)}{(2\ell+1)^2} \quad\text{and}\quad 
\beta_k=\sum_{(\ell,m):k=(2\ell+1)m} b_m\cdot(-1)^{\ell}\frac{\mu(2\ell+1)}{(2\ell+1)^2}.
\end{gather*}


We now show that 
$(\alpha_k\cdot k^s)_{k=1}^\infty$ is square summable if $0\le s<1$ and
$a^s:=(a_m\cdot m^s)_{m=1}^\infty$ is square summable. 
We rewrite
\begin{align}
\notag\sum_{k=1}^\infty \alpha_k^2\,k^{2s}&=\sum_{k=1}^\infty k^{2s}\sum_{(\ell,m):k=(2\ell+1)m}a_m\frac{\mu(2\ell+1)}{(2\ell+1)^2}\cdot \sum_{(\ell',m'):k=(2\ell'+1)m'}a_{m'}\frac{\mu(2\ell'+1)}{(2\ell'+1)^2}\\
\notag &=\sum_{\substack{k,\ell,m,\ell',m'\\k=(2\ell+1)m\\k=(2\ell'+1)m'}} k^{2s} a_m\frac{\mu(2\ell+1)}{(2\ell+1)^2}\cdot a_{m'}\frac{\mu(2\ell'+1)}{(2\ell'+1)^2}\\
\notag &=\sum_{m,m'=1}^\infty a_m a_{m'} m^s (m')^s\sum_{\substack{k,\ell,\ell'\\k=(2\ell+1)m\\k=(2\ell'+1)m'}} \frac{\mu(2\ell+1)}{(2\ell+1)^{2-s}}\cdot \frac{\mu(2\ell'+1)}{(2\ell'+1)^{2-s}}\\
\label{eq:aXa}&=\sum_{m,m'=1}^\infty a_m m^s \cdot a_{m'}(m')^s X_{m,m'}=\langle a^s,X a^s\rangle,
\end{align}
where
\begin{equation}\label{eq:defX}
X_{m,m'}=\sum_{\substack{\ell,\ell'\\(2\ell+1)m=(2\ell'+1)m'}}\frac{\mu(2\ell+1)}{(2\ell+1)^{2-s}}\cdot\frac{\mu(2\ell'+1)}{(2\ell'+1)^{2-s}}.
\end{equation}

We aim to show that $X=(X_{m,m'})_{m,m'=1}^\infty$ is a bounded operator on $\ell_2$, which together with
\eqref{eq:aXa}, finishes the proof.
We shall use the symmetry of $X$ and the following simple observation. If $\sum_{m'}|X_{m,m'}|$ is uniformly bounded over $m$, then $X$ generates
a bounded linear operator on $\ell_1$ and also on $\ell_\infty$ and,
by interpolation, also on $\ell_2$.

We use that, for any $m\ge 1$ fixed,
\begin{align*}
\sum_{m'=1}^\infty |X_{m,m'}|\le \sum_{\ell,\ell'=1}^\infty \frac{|\mu(2\ell+1)|}{(2\ell+1)^{2-s}}\cdot\frac{|\mu(2\ell'+1)|}{(2\ell'+1)^{2-s}}\sum_{m':(2\ell+1)m=(2\ell'+1)m'}1.
\end{align*}
The last sum is (at most) 1 and the sum over $\ell$ and $\ell'$ is convergent.

\emph{Step 2.} The proof of $\F^s([0,1]) \hookrightarrow W^s([0,1])$ for $0\le s<1$
follows the same pattern, only now we assume that $f$ is given by \eqref{eq:decomp_f_CalC}
and invoke the decomposition of $\CalC$ and $\CalS$ into their respective Fourier series, cf. \cite[p. 166]{DD+22}.
As the coefficients of these series decay again quadratically (similarly to \eqref{eq:cos_Riesz} and \eqref{eq:sin_Riesz}),
the rest of the proof follows in the same manner.

\emph{Step 3.}  We finally show that \eqref{eq:Riesz_1} is a Riesz basis of $W^s$. 
For this, note that the above shows 
$\|f\|_{W^s}\approx\|f\|_{\F^s}$, 
where ``$\approx$'' means upper and lower bounded with constants that are independent of $f$. In particular, 
\[
\left\|\sum_{k\in\N} k^{-s} \Bigl[\alpha_k \CalC_k + \beta_k \CalS_k\Bigr]\right\|^2_{W^s}
\approx 
\left\|\sum_{k\in\N} k^{-s} \Bigl[\alpha_k \CalC_k + \beta_k \CalS_k\Bigr]\right\|^2_{\F^s}=
\sum_{k\in\N} k^{2s} \, \frac{\alpha_k^2+\beta_k^2}{k^{2s}}, 
\]
proving the claim.
\end{proof}

\medskip

\begin{rem} 
We comment on the restriction to $0\le s<1$ in Theorem \ref{thm:d1_WF}.


If $s\ge 3/2$, then \eqref{eq:cos_Riesz} implies that $\cos(2\pi x)$ does not lie in $\F^s([0,1])$ and, similarly,
one can also show that $\CalC\not\in W^s([0,1])$. Therefore, Theorem \ref{thm:d1_WF} cannot hold if $s\ge 3/2$.


The case $1\le s<3/2$ is more delicate and we conjecture that also in this case Theorem \ref{thm:d1_WF} fails.
Observe, that the proof of Theorem \ref{thm:d1_WF} reduces the study of the embedding
$W^s([0,1])\hookrightarrow \F^s([0,1])$ to the boundedness of the matrix $X=(X_{m,m'})_{m,m'=1}^\infty$
with $X_{m,m'}$ given by \eqref{eq:defX} on $\ell_2.$ To rewrite \eqref{eq:defX} for fixed $m,m'\in\N$,
let $g=\operatorname{gcd}(m,m')$ be the greatest common divisor of $m$ and $m'$. For simplicity, assume for the moment
that $m$ and $m'$ are odd. Then all the integer solutions of the equation $(2\ell+1)m=(2\ell'+1)m'$
are given by $2\ell+1=(2t+1)\cdot \frac{m'}{g}$ and $2\ell'+1=(2t+1)\cdot \frac{m}{g}$, where $t\in\N_0$ is arbitrary.
Therefore, we obtain
\begin{align*}
X_{m,m'}&=\sum_{t=0}^\infty
\frac{\mu((2t+1)\cdot\frac{m'}{g})\mu((2t+1)\cdot\frac{m}{g})}{((2t+1)\cdot \frac{m'}{g})^{2-s}((2t+1)\cdot \frac{m}{g})^{2-s}}\\
&=\frac{\operatorname{gcd}(m,m')^{4-2s}}{m^{2-s}{m'}^{2-s}}
\sum_{t=0}^\infty\frac{\mu((2t+1)\cdot\frac{m'}{g})\mu((2t+1)\cdot\frac{m}{g})}{(2t+1)^{4-2s}}.
\end{align*}
The boundedness of the infinite matrix $Y=(\operatorname{gcd}(m,m')^{4-2s}/[m\cdot m']^{2-s})_{m,m'=1}^\infty$
was investigated in \cite{LK}, \cite{HLS}, and \cite[page 578]{W} and it is known to fail if $3/2>s\ge 1$.
Therefore, we believe that also $X$ is not bounded on $\ell_2$ for this range of $s$.\\
\end{rem}

\begin{rem}
Finally, let us remark that by using similar techniques as above we can also prove the embedding 
\begin{equation*}
W^{ s }([0,1])\hookrightarrow \F^{\gamma}([0,1])\qquad \text{for all }\quad   \gamma<\min\left\{s,\frac{1+s}{2},\frac{3}{2}\right\},  
\end{equation*}
which, in particular, implies 
$W^{ 1 }([0,1])\hookrightarrow \F^{1-\delta}([0,1])$ 
for all $\delta>0$.
We omit the details.  
\end{rem}

\goodbreak

\subsection{Multivariate Sobolev classes}
We now show how the results of the previous section can be extended to higher dimensions.
In general, we follow the scheme presented for the univariate case.
To simplify the notation, we replace the orthonormal basis of dilated $\cos$ and $\sin$ functions from \eqref{eq:decomp_f_Fourier}
by the usual orthonormal basis of exponentials.

We consider the complex-valued $d$-variate Fourier series
\begin{equation}\label{eq:f_Four}
f(x)=\sum_{m\in\Z^d} a_m e^{2\pi i m\cdot x},\quad x\in[0,1]^d
\end{equation}
and define for every $s\ge 0$ the Sobolev space
\[
W^s([0,1]^d):=\left\{f: f\text{ given by \eqref{eq:f_Four}
and }\|f\|^2_{W^s}=
|a_0|^2+\sum_{m\in\Z^d\setminus\{0\}}\|m\|^{2s}_2\cdot |a_m|^2<\infty\right\}.
\]
As mentioned already in the introduction, it was discovered in \cite{SV}, that 
$\cR_d$ from~\eqref{eq:intro_Riesz_d}
is a Riesz basis of $L_2([0,1]^d)$ for every $d\ge 2$
and that the constants in the Riesz-type estimate \eqref{eq:Riesz}  can be chosen independently on $d$. The aim of this section is to extend this result
also to $W^s([0,1]^d)$ for $s$ small enough.

Recall that $\CalC_k(x)\,:=\,\CalC(k\cdot x)$ and $\CalS_k(x)\,:=\,\CalS(k\cdot x)$ and consider the decomposition 
\begin{equation}\label{eq:f_Riesz_d}
f(x)\,=\,\alpha_0+\sum_{k\,\righttriangleplus\,0}
\Bigl[\alpha_k\CalC_k(x)+\beta_k\CalS_k(x)\Bigr],
\quad x\in [0,1]^d. 
\end{equation}
Recall that $k\,\righttriangleplus\, 0$
means that $k=(k_1,\dots,k_d)\in\Z^d$ is not equal to zero and the first non-zero entry of $k$ is positive.
We define for $s\ge 0$ the spaces 
\[
\F^s([0,1]^d):=\left\{f: f\text{ is given by \eqref{eq:f_Riesz_d} and }\|f\|^2_{\F^s}:=\alpha_0^2+\sum_{k\,\righttriangleplus\,0}\|k\|^{2s}_2(\alpha_k^2+\beta_k^2)<\infty\right\}.
\]

We now present the multivariate analogue of Theorem~\ref{thm:d1_WF}. 

\begin{thm}\label{thm:FsWs_d}  
Let $0\le s < 1$. Then, $W^s([0,1]^d)= \F^s([0,1]^d)$ in the sense of equivalent norms.
The constants of equivalence of these norms depend on $s$ but are independent of~$d$.
Moreover, the system
\begin{equation}\label{eq:Riesz_d-s}
\{1\}\cup    \left\{\|k\|_2^{-s}\CalC_k,\, \|k\|_2^{-s}\CalS_k\colon k\in \Z^d,\, k\,\righttriangleplus\, 0 \right\}
\end{equation}
is a Riesz basis of $W^s([0,1]^d)$ and the constants $c,C>0$ in 
\begin{equation}\label{ineq-Riesz-d}
c\,\sum_{k\righttriangleplus 0}
(\alpha_k^{2}+\beta_k^2) 
\;\le\; 
\left\|\sum_{k\righttriangleplus 0}\|k\|_2^{-s}\Bigl(\alpha_k\CalC_k+ 
\beta_k \CalS_k\Bigr)\right\|^2_{W^s}
\;\le\; C\,\sum_{k\righttriangleplus 0}
(\alpha_k^{2}+\beta_k^2) 
\end{equation}
can be chosen independently of $d$. \\
(Again, modifications of \eqref{ineq-Riesz-d} for functions with constant term are obvious.)
\end{thm}

\begin{proof}[Proof of Theorem~\ref{thm:FsWs_d}.] We proceed similarly as in the proof of Theorem \ref{thm:d1_WF}.

\medskip

\emph{Step 1.} First, we show that $W^s([0,1]^d)\hookrightarrow \F^s([0,1]^d)$.

Let $f$ be given by \eqref{eq:f_Four}, i.e., let 
\begin{align*}
f(x)=\sum_{m\in\Z^d} a_m e^{2\pi i m\cdot x}=a_0+\sum_{m\,\righttriangleplus\, 0} b_m \cos(2\pi m\cdot x) + \sum_{m\,\righttriangleplus\, 0} b'_m\sin(2\pi m\cdot x),
\end{align*}
where $b_m=a_m+a_{-m}$ and $b'_m=i(a_m-a_{-m})$. We estimate only the first sum, the second can be treated similarly. Using \eqref{eq:cos_Riesz} we obtain
\begin{align*}
\sum_{m\,\righttriangleplus\, 0} b_m \cos(2\pi m\cdot x)&=
\frac{\sqrt{3}}{\sqrt{2}\, \kappa}\sum_{{m\,\righttriangleplus\, 0}} b_m \sum_{\ell=0}^\infty \frac{\mu(2\ell+1)}{(2\ell+1)^2}\cdot {\mathcal C}((2\ell+1)m\cdot x)\\
&=\frac{\sqrt{3}}{\sqrt{2}\, \kappa}\sum_{{k\,\righttriangleplus\, 0}}\alpha_k\CalC(k\cdot x),
\end{align*}
where
\[
\alpha_k=\sum_{(\ell,m):(2\ell+1)m=k}b_m \cdot \frac{\mu(2\ell+1)}{(2\ell+1)^2}.
\]
From this we deduce
\begin{align*}
\sum_{k\,\righttriangleplus\, 0} \alpha_k^2\,\|k\|_2^{2s}&=\sum_{k\,\righttriangleplus\, 0} \|k\|_2^{2s}\sum_{(\ell,m):k=m(2\ell+1)}b_m\frac{\mu(2\ell+1)}{(2\ell+1)^2}\cdot \sum_{(\ell',m'):k=m'(2\ell'+1)}b_{m'}\frac{\mu(2\ell'+1)}{(2\ell'+1)^2}\\
&=\sum_{\substack{k,\ell,m,\ell',m'\\k=m(2\ell+1)\\k=m'(2\ell'+1)}} \|k\|_2^{2s} b_m\frac{\mu(2\ell+1)}{(2\ell+1)^2}\cdot b_{m'}\frac{\mu(2\ell'+1)}{(2\ell'+1)^2}\\
&=\sum_{m,m'\,\righttriangleplus\, 0} b_m \|m\|_2^s \cdot b_{m'}\|m'\|_2^s X_{m,m'},
\end{align*}
where
\begin{align*}
X_{m,m'}&=\sum_{\substack{\ell,\ell'\\(2\ell+1)m=(2\ell'+1)m'}}
\frac{\|(2\ell+1)m\|_2^s}{\|m\|_2^s}\cdot\frac{\|(2\ell'+1)m'\|_2^s}{\|m'\|_2^s}\cdot\frac{\mu(2\ell+1)}{(2\ell+1)^2}\cdot\frac{\mu(2\ell'+1)}{(2\ell'+1)^2}\\
&=\sum_{\substack{\ell,\ell'\\(2\ell+1)m=(2\ell'+1)m'}}\frac{\mu(2\ell+1)}{(2\ell+1)^{2-s}}\cdot\frac{\mu(2\ell'+1)}{(2\ell'+1)^{2-s}}.
\end{align*}

We finish the proof by showing that $X$ has uniformly bounded row sums and, therefore, is bounded on $\ell_2$.
Indeed,
\begin{align*}
\sum_{m'\,\righttriangleplus\, 0} |X_{m,m'}|\le \sum_{\ell,\ell'=0}^\infty \frac{|\mu(2\ell+1)|}{(2\ell+1)^{2-s}}\cdot\frac{|\mu(2\ell'+1)|}{(2\ell'+1)^{2-s}}\sum_{m':(2\ell+1)m=(2\ell'+1)m'}1,
\end{align*}
which is again convergent for $0\le s<1$, and the norm is uniformly bounded in $d\ge 1.$

\medskip

\emph{Step 2.} The inverse embedding $\F^s([0,1]^d)\hookrightarrow W^s([0,1]^d)$ follows in a very similar way. We assume that
$f$ is given by \eqref{eq:f_Riesz_d}, decompose $\CalC$ and $\CalS$ into their respective Fourier series and reduce the embedding to the boundedness
of a certain infinite matrix on $\ell_2$. As it resembles very much the matrix $X$ used above, the arguments are  very similar. 

\medskip

\emph{Step 3.} 
The proof of \eqref{ineq-Riesz-d} is analogous to Step~3 of the proof of Theorem~\ref{thm:d1_WF}. We leave out the details.
\end{proof}


\medskip

\subsection{Barron classes}

Barron classes, 
which were first introduced in \cite{Barron},  turned out to be a very promising function class for the performance analysis  
of artificial neural networks for solving  high-dimensional problems. In particular, if $\mu$ is a probability measure on a $d$-dimensional ball
and $f$ is a function from the 'original' Barron class 
(see \eqref{Barron-orig} below with $s=1$), \cite{Barron} provides a randomized construction of a shallow
neural network with one hidden layer of $N$ neurons, which achieves an approximation of $f$ with accuracy of the order $N^{-1/2}$
in the $L_2(\mu)$-norm. 
Remarkably, this approximation rate is independent of the dimension $d$ even though the Barron class is so large that every linear approximation method for it suffers
the curse of dimensionality \cite[Theorem~6]{Barron}.  
These results were later extended and generalized in various directions, see  e.g. \cite{CPV23,EW20} and the references given therein.

We start again by introducing the 
(Fourier-analytic version of) Barron classes 
and an alternative based on the 
Riesz basis $\cR_d$ of $L_2$. 
We will see below that, again, 
those spaces coincide with equivalent norms.
Let $s\ge 0$ and define
\[
{\mathbb B}^s([0,1]^d):=\left\{f: f\text{ given by \eqref{eq:f_Four}, }
\|f\|_{{\mathbb B}^s}:=|a_0|+\sum_{m\in\Z^d\setminus{\{0\}}}\|m\|_2^s\cdot |a_m|<\infty\right\}, 
\]
and 
\[
{\mathcal B}^s([0,1]^d):=\left\{f: f\text{  given by \eqref{eq:f_Riesz_d}, }
\|f\|_{{\mathcal B}^s}:=|\alpha_0|+\sum_{k\,\righttriangleplus\,0}\|k\|_2^s\cdot (|\alpha_k|+|\beta_k|)<\infty\right\}.
\]

Note that in contrast to the Sobolev classes from the previous section, we now invoke the $\ell_1$-norm instead of the $\ell_2$-norm when defining the Barron classes.
For a discussion of other notions of Barron spaces  we refer to  Remark~\ref{rem:different-B} below. 

\medskip
\goodbreak

We now show that the Barron-type spaces above have equivalent norms.

\begin{thm}\label{thm:Barron}
Let $0\le s < 1$. Then ${\mathbb B}^s([0,1]^d) = {\mathcal B}^s([0,1]^d)$ in the sense of equivalent norms.
The constants of equivalence of these norms depend on $s$ but are independent of~$d$. \\
Moreover, it holds that 
\begin{equation}\label{eq:Riesz-Barron}
c\,\sum_{k\righttriangleplus 0}
(|\alpha_k|+|\beta_k|) 
\;\le\; 
\left\|\sum_{k\righttriangleplus 0}\|k\|_2^{-s}\Bigl(\alpha_k\CalC_k+ 
\beta_k \CalS_k\Bigr)\right\|_{\mathbb{B}^s}
\;\le\; C\,\sum_{k\righttriangleplus 0}
(|\alpha_k|+|\beta_k|) 
\end{equation}
where the constants $c,C>0$ can be chosen independently of $d$. \\
(Modifications of \eqref{eq:Riesz-Barron} for functions with constant term are obvious.)
\end{thm}

Some authors call  the 
system \eqref{eq:Riesz_d-s} 
a $1$-Riesz basis of $\mathbb{B}^s$ 
if it satisfies~\eqref{eq:Riesz-Barron}, 
see e.g.~\cite{CS03}.

\begin{proof} \emph{Step 1.} First, we show that ${\mathbb B}^s([0,1]^d)\hookrightarrow {\mathcal B}^s([0,1]^d)$.

We proceed similarly as in the proof of Theorem \ref{thm:FsWs_d}. Let us provide a  few details. Let $\omega=(\omega_m)_{m\,\righttriangleplus\,0}$
be a sequence given by $\omega_m=\|m\|_2^s$. The norm of the sequence $\alpha=(\alpha_m)_{m\,\righttriangleplus\,0}$
in the weighted space $\ell_1(\omega)$ is then given as usual by
\[
\|\alpha\|_{\ell_1(\omega)}=\sum_{m\,\righttriangleplus\,0} \omega_m\cdot |\alpha_m|.
\]

Repeating the argument of the proof of Theorem \ref{thm:FsWs_d}, we reduce 
the proof of the embedding ${\mathbb B}^s([0,1]^d)\hookrightarrow {\mathcal B}^s([0,1]^d)$
to the estimate of the operator norm of $T:\ell_1(\omega)\to\ell_1(\omega)$,
where, for $k\righttriangleplus 0$, 
\begin{align*}
\alpha_k=(Ta)_k=\sum_{\substack{m\righttriangleplus 0; \ell=0,1,\dots\\(2\ell+1)m=k}}a_m\cdot \frac{\mu(2\ell+1)}{(2\ell+1)^2}.
\end{align*}
Due to convexity, the norm of $T$ is attained on the extreme points of the unit ball of $\ell_1(\omega)$, i.e., 
\[
\|T\|_{\ell_1(\omega)\to \ell_1(\omega)}=\sup_{m\,\righttriangleplus\,0}\frac{\|T e_m\|_{\ell_1(\omega)}}{\|e_m\|_{\ell_1(\omega)}}=
\sup_{m\,\righttriangleplus\,0}\frac{\|T e_m\|_{\ell_1(\omega)}}{\|m\|_2^s}.
\]
If $m\in\Z^d, m\,\righttriangleplus\,0$ is fixed, then
\begin{align*}
   (Te_m)_k&=\sum_{(\ell,n):(2\ell+1)n=k}(e_m)_n\cdot\frac{\mu(2\ell+1)}{(2\ell+1)^2}\\
   &=\sum_{\substack{\ell=0,1,\dots\\(2\ell+1)m=k}}\frac{\mu(2\ell+1)}{(2\ell+1)^2}=
   \begin{cases} 0,&\text{if there is no $\ell\ge 0$ with}\ (2\ell+1)m=k,\\
   \frac{\mu(2\ell+1)}{(2\ell+1)^2},&\text{if} \ (2\ell+1)m=k.
   \end{cases} 
\end{align*}
Hence,
\begin{align}
\notag \|m\|_2^{-s} \|Te_m\|_{\ell_1(\omega)}&=\|m\|_2^{-s}\sum_{\ell=0}^\infty \|(2\ell+1)m\|^s_2\cdot \frac{|\mu(2\ell+1)|}{(2\ell+1)^2}\\
\label{eq:Barron_series} &=\sum_{\ell=0}^\infty \frac{|\mu(2\ell+1)|}{(2\ell+1)^{2-s}},
\end{align}
which is finite for $0\le s<1$ and, of course, independent on $d$.

\emph{Step 2.} The reverse embedding ${\mathcal B}^s([0,1]^d) \hookrightarrow {\mathbb B}^s([0,1]^d)$
can be shown in nearly the same way. 

\emph{Step 3.}  The proof of~\eqref{eq:Riesz-Barron} follows again Step~3 of the proof of Theorem~\ref{thm:d1_WF}.\\
\end{proof}

\begin{rem}
It can be shown that the series \eqref{eq:Barron_series} converges only for $s<1$. Indeed, by \cite[Equation (1.2.7)]{Tit}, it is known that the series
\[
\sum_{n=1}^\infty \frac{|\mu(n)|}{n^\alpha}
\]
converges only for $\alpha>1$. On the other hand,
\begin{align*}
\sum_{n=1}^\infty \frac{|\mu(n)|}{n^\alpha}&=\sum_{n=0}^\infty \frac{|\mu(2n+1)|}{(2n+1)^\alpha}+\sum_{n=0}^\infty \frac{|\mu(4n+2)|}{(4n+2)^\alpha}
+\sum_{n=0}^\infty \frac{|\mu(4n+4)|}{(4n+4)^\alpha}\\
&=\Bigl(1+\frac{1}{2^\alpha}\Bigr)\sum_{n=0}^\infty \frac{|\mu(2n+1)|}{(2n+1)^\alpha},
\end{align*}
where we used that $4(n+1)$ is divisible by 4 implying $\mu(4n+4)=0$, and $\mu(4n+2)=-\mu(2n+1)$ for every $n\ge 0$. Therefore, \eqref{eq:Barron_series} converges only when $2-s>1$
and the restriction to $0\le s<1$ in Theorem \ref{thm:Barron} is necessary. 
\end{rem}

\begin{rem}\label{rem:different-B}
Let us comment on the different notions of Barron spaces in the literature.  
Our Barron spaces ${\mathbb B}^s$ are in the spirit of the ``classical'' Fourier-analytic notion of Barron spaces of periodic functions as introduced in~\cite{Barron}, 
see Example~16 in Section~IX, 
at least for $s=1$. 
However, it seems more common to work with the 
classes of functions on $\R^d$, 
or on certain subsets $\Omega$, 
which are defined by 

\begin{equation}\label{Barron-orig}
\mathcal{B}_{\mathrm{ext}}^s(\Omega) 
:= \left\{f:\Omega\rightarrow \mathbb{C}: \ 
\inf_{f_e|_{\Omega}=f}\int_{\real^d}(1 +\|\xi\|_2)^s|\hat{f}_e(\xi)|d\xi<\infty\right\}, 
\end{equation}
where the infimum is taken over all extensions $f_e\in L_1(\real^d)$ and $\hat{f}_e$ denotes the Fourier transform of $f_e$.
Barron \cite{Barron} originally introduced  this class for $s=1$, 
and showed that functions from this class can be efficiently approximated by neural networks. 
However, we stress that the  norm 
\eqref{Barron-orig} is dimension-dependent in a non-transparent way: 
It is already a non-trivial research question to determine the norm of a constant function (for bounded~$\Omega$). 
See also~\cite{CPV23,Siegel21} for 
a discussion of issues with general domains.

Another notion of Barron-type spaces that has been proposed in the literature 
(see e.g.~\cite{EMW22,E1})
is the space of all “infinitely wide”
neural networks with a certain control over the network parameters. More formally, given an
activation function $\sigma$ (which is often either the ReLU or a Heaviside function), the elements of the associated Barron-type space are all functions that can be written as
\[
f(x)=\int_{\real\times \real^d\times \real}
 a\cdot  \sigma(\langle w, x\rangle + b) \, d\mu (a, w, b),
\]
where $\mu$ is a probability measure satisfying 
\begin{equation}\label{Barron-new}
\|f\|_{B(\sigma)}:=\int_{\real\times \real^d\times \real}
 |a|\cdot  \sigma(\|w\|_2 + |b|) \, d\mu (a, w, b)<\infty. 
\end{equation}
Note that this space/norm is automatically adapted to the activation function $\sigma$. 
It was further generalized in \cite{Per24}, where $\ell_p$ norms and $\mathrm{ReLU}^k$ activation functions were considered in \eqref{Barron-new}. 
Let us point here  that already in \cite{Barron} it was shown that 
the Fourier-analytic Barron space according to \eqref{Barron-orig} is
embedded in the Barron space $B(H)$ associated to the Heaviside function $H$, i.e., 
for bounded domains $\Omega$ it holds that 
$\mathcal{B}_{\mathrm{ext}}^1(\Omega) \hookrightarrow B(H)$. On the other hand, 
when considering  the ReLU activation function $\sigma$, 
we only have 
$\mathcal{B}_{\mathrm{ext}}^2(\Omega) \hookrightarrow B(\sigma)$, 
see~~\cite[Lemma~7.1]{CPV23}.
The latter embedding is best possible  as was shown in~\cite[Prop.~7.4]{CPV23}. 
We also refer to \cite[Section~2]{EMW22}, \cite[Section~3]{E1} or~\cite{KB} 
in this context.
\end{rem}

\section{Approximation of functions by ANNs}\label{sec:ANN}

In this section we  exploit the properties of the Riesz basis \eqref{eq:Riesz_d-s} derived in Section \ref{Sec:2} to address the question  how to approximate the 
elements of the Sobolev classes $W^s([0,1]^d)$ and Barron classes $\mathbb{B}^s([0,1]^d)$  with $0<s<1$  by artificial neural networks. 
The main results of this section are stated in Theorems \ref{thm:ANN_Ws} and \ref{thm:ANN_Bs}.

We first fix some  notation.  
A function $f:\R^{n_1}\to \R^{n_2}$ is called affine, if it can be written as $f(x)=Mx+b$, where $M\in\R^{n_2\times n_1}$
is a matrix and $b\in\R^{n_2}.$ This means that $f$ is given by $(n_1+1)n_2$ real parameters.
In the setting of artificial neural networks, $M$ is usually called the weight matrix and $b$ the bias vector.
The following definition formalizes the notion
of $\ReLU$ neural networks with width $W$ and depth $L$, cf. Figure \ref{fig:NN}. 
For further information regarding the basics of neural network architectures we refer to the survey article \cite{DeVore-2} as well as the recent book \cite{Petersen-book} and the references therein.

\begin{dfn} 
Let $d,W,L$ be positive integers. Then a feed-forward artificial neural network ${\mathcal N}$ with the $\ReLU$ activation function 
width $W$, 
and depth $L$ is given by a collection of $L+1$ affine mappings $A^{(0)},\dots,A^{(L)}$, where $A^{(0)}:\R^d\to \R^W$,
$A^{(j)}:\R^W\to\R^W$ for $j=1,\dots,L-1$ and $A^{(L)}:\R^W\to \R$, 
and generates
a function $\mathcal{N}\colon\R^d\rightarrow \R$ 
of the form 
\[
A^{(L)}\circ \ReLU\circ A^{(L-1)}\circ\cdots\circ\ReLU\circ A^{(0)}.
\]
We denote  by $\Upsilon_d^{W,L}$ the set of all such ANNs.
\end{dfn}

\begin{figure}[h!]
\begin{center}\includegraphics[width=9cm]{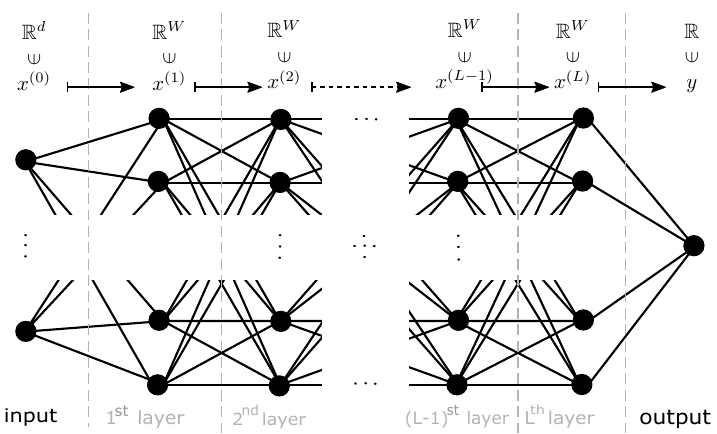}\\
\caption{Feed-forward $\mathrm{ReLU}$ network with length $L$ and width $W$}
\label{fig:NN}
\end{center}
\end{figure}


Counting all the weights and biases of a fully connected network from $\Upsilon^{W,L}_d$, we see that it has
\begin{equation} \label{eq:number}
W(d+1)+(L-1)W(W+1)+W+1=\mathcal{O}(W^2L+dW)
\end{equation}
parameters. When $d=1$ this reduces to $\mathcal{O}(W^2 L)$, see also  \cite[Remark~6.1]{DD+22}. 
For an ANN, we call the graph of the connections between its artificial neurons (as shown in Figure \ref{fig:NN} for $\Upsilon^{W,L}_d$) its \emph{architecture}.
In general, an ANN could have a different number of artificial neurons in each layer, but we restrict ourselves to the architecture of $\Upsilon^{W,L}_d$
to simplify the presentation. 
It is obvious that every function $\cN$ generated by such a  feed-forward $\ReLU$ network is a continuous piecewise affine function on ${\mathbb R}^d$.
When approximating  functions from the Sobolev or Barron classes we now use the fact that these functions admit a representation of the form \eqref{eq:f_Riesz_d}, i.e., 
\[ 
f(x)=\alpha_0+\sum_{k\,\righttriangleplus\,0}\Bigl[\alpha_k\CalC_k(x)+\beta_k\CalS_k(x)\Bigr],\quad x\in [0,1]^d.
\]
Moreover, as was shown in \cite[Theorem~6.2]{DD+22} for $d=1$ and \cite[Theorem~4.6]{SV} for higher dimensions, one can reproduce linear combinations   of  $\CalC_k$ and $\CalS_k$ via $\ReLU$ networks with a good control of the depth $L$. This makes the representation \eqref{eq:f_Riesz_d} an ideal starting point for approximating $f$ using neural networks.

Let us briefly sketch the construction. 
The basic idea is to start with a representation of the hat function $H:[0,1]\rightarrow \R$ using the $\ReLU$ function, i.e., 
\begin{equation}
H(x)=\left.\begin{cases}\label{hat-repr}
2x,& 0\leq x\leq \frac 12\\
2(1-x), & \frac 12<x\leq 1
\end{cases}\right\}\ = \ (2, -4)\,\cdot\,  \ReLU\left(\begin{pmatrix}1\\1\end{pmatrix}x+\begin{pmatrix}0\\ -\frac12\end{pmatrix}\right).
\end{equation}
From the right hand side of \eqref{hat-repr} we deduce that  $H$ is the realization of an ANN with width $W=2$ and one hidden layer $L=1$, 
as illustrated in Figure \ref{fig:hat}, 
i.e., $H\in\Upsilon_1^{2,1}$. 

\begin{figure}[h!]
\begin{center}\includegraphics[width=11cm]{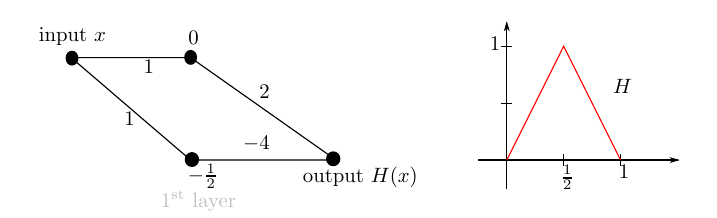}
\caption{ANN representing $H$ and usual graph associated with $H$}
\label{fig:hat}
\end{center}
\end{figure}

We will realize our Riesz basis $\mathcal{R}_d$ by compositions and sums of these simple networks. 
For this, let us first mention that for
functions that are represented by ANNs, say, with width 
$W_i$ and depth $L_i$, $i=1,\dots,k$, 
we can represent their composition by an ANN 
with width $\max_i\{W_i\}$ and depth $\sum_{i=1}^k L_i$, 
by just concatenating the ANNs.


Now, $\CalC=1-2H$, with $\CalC$ from \eqref{C-intro}, can be written as 
\[
\CalC(x) \,=\, (-4, 8)\,\cdot\,  \ReLU\left(\begin{pmatrix}1\\1\end{pmatrix}x+\begin{pmatrix}0\\ -\frac12\end{pmatrix}\right) 
\,+\, 1,
\]
which has the same architecture as $H$ with 
width $W=2$ and depth $L=1$. 
Using $\CalS(x)=\CalC_2(x/2+3/8)$, 
we see that $\CalS$ can be represented using the ANN architecture for $\CalC_2$, i.e., with $W=L=2$. 
The univariate scaled versions $\CalC_k,\CalS_k$ are obtained using compositions of hat functions, e.g., it holds $\CalC_k(x)=\CalC(H(\ldots (H(x))))$ for $k=2^m$ where $H$ is composed $m$-times. 
The general, multivariate case is obtained by a suitable scaling argument 
that ultimately leads to 
$\CalC_k,\CalS_k\in\Upsilon_d^{2,L}$ 
with $L:=\log_2\left(\|k\|_1 \right)+4$, 
$k\in\Z^d$, 
see~\cite[Lemma~4.4]{SV}. 

For linear combinations of these functions it remains to sum the individual terms,  
and in this matter 
we have some freedom. 
Since the individual terms can be computed separately, 
we could stack them on top of each other 
leading to an ANN with width $\sum_{i=1}^N W_i$
and depth $\max_i\{L_i\}$. 
However, we could also do the summation 
\textit{in line}, as in Figure~\ref{fig:NN-sum}, by using two extra channels (a source   and a collation channel). 
This leads to an ANN with width $\max_i\{W_i\}+d+1$ and depth $\sum_{i=1}^N L_i$. 
One may also mix both strategies, by arranging the summands in an array. 
We refer to \cite[Theorem~6.2]{DD+22}   and \cite[Theorem~4.6]{SV} for further details on the constructions.

\begin{figure}[h!]
\begin{center}\includegraphics[width=15.5cm]{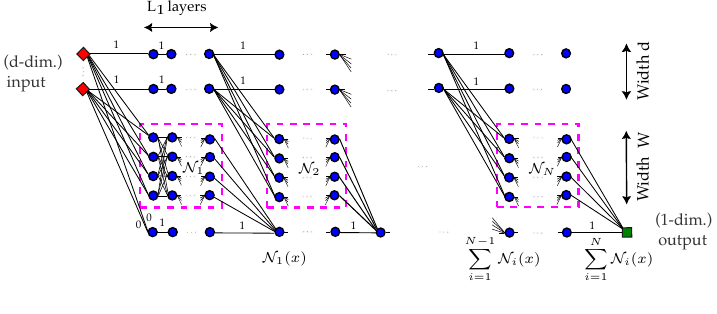}
\caption{ANN representing the sum of   functions $\mathcal{N}_1, \ldots, \mathcal{N}_N$ \\
}
\label{fig:NN-sum}
\end{center}
\end{figure}

This results in the following architecture for 
partial sums of~\eqref{eq:f_Riesz_d}, 
see \cite[Theorem~4.6]{SV}.

\begin{lem} \label{lem:ANN}
Let $d,R\ge1$ and let $I\subset\Z^d\setminus\{0\}$ be nonempty. 
Then, 
\begin{equation}\label{eq:fdecomp}
f(x)\,=\,\alpha_0
+\sum_{k\in I} \Bigl[\alpha_k\CalC_k(x)+\beta_k\CalS_k(x)\Bigr],\quad x\in [0,1]^d, 
\end{equation}
satisfies $f\in\Upsilon_d^{W,L}$ with either 
\begin{equation}\label{constr-1}
W=4\cdot \#I
\qquad\text{and}\qquad 
L\,\le\, 4+\log_2\left(\max_{k\in I} \|k\|_1 \right),
\end{equation}
or
\begin{equation}\label{constr-2}
\qquad
W=d+3 
\qquad\quad\text{and}\qquad 
L \,\le\, 2\,\#I \cdot \log_2\left(16\cdot\max_{k\in I} \|k\|_1 \right). 
\end{equation}
Moreover, all weights in the corresponding ANN are bounded by $8\cdot\max_{k\in I}\{1,|\alpha_k|,|\beta_k|\}$.
\end{lem}

\begin{rem} 
The decomposition \eqref{eq:fdecomp} features $2\,\#I+1$ real parameters, namely $\alpha_0, (\alpha_k)_{k\in I}$ and $(\beta_k)_{k\in I}$.
The ANNs constructed in Lemma \ref{lem:ANN} combine these parameters with pre-cached building blocks (i.e., ANNs, which reproduce the functions $\CalC_k$ and $\CalS_k$ for $k\in I$)
and additional information passing channels.
In this sense, these $2\,\#I+1$ parameters are the only parameters, which we change in the proposed architecture,
the other weights are fixed and do not depend on $f$. This approach was used in \cite[Section VII]{EPGB} under the name \emph{transference principle}
for general decompositions of functions into a given \emph{dictionary}. However, in the next section we take  advantage of the following two facts: 
First, we exploit that the elements of the Riesz basis \eqref{eq:Riesz_d-s} can be easily and exactly recovered by small ANNs. And secondly,
the elements of \eqref{eq:Riesz_d-s} do not rely on any localized decomposition of unity, which would lead to exponential dependence of the constants on
the underlying dimension $d$.\\ 
If we count all nonzero parameters of the architecture, 
then we can employ the fact that the widths and depth of the 
small ANNs are at most 4 and $M:=1+\log_2[\max_{k\in I} \|k\|_1]$, respectively. Hence, there are at most $\mathcal{O}(\#I\cdot\max(d,M))$ non-zero weights in the ANN from the first part of Lemma~\ref{lem:ANN}, 
and $\mathcal{O}(d\cdot \#I\cdot M)$ non-zero weights in the second architecture. 

In contrast, we know from~\eqref{eq:number} that 
the number of parameters in a (general, fully-connected) ANN from $\Upsilon_d^{W,L}$, with $W$ and $L$ chosen as above, 
is of the order $\mathcal{O}((\#I)^2\cdot M+\#I\cdot d)$ 
for the first, and $\mathcal{O}(d^2\cdot \#I\cdot M)$ for the second architecture. 
Note that in the second construction, 
most of the weights are actually equal to 1 due to the additional collation channel.

To ease the presentation, 
we will only employ the first construction \eqref{constr-1} 
for the approximation results below.  
The modifications for the second construction \eqref{constr-2} are straightforward. 
\end{rem}

\bigskip

\subsection{Sobolev classes}

We start with approximation of functions from Sobolev classes. We use Theorem \ref{thm:FsWs_d} to show that
for every $d\ge 1$, $0<s<1$, and $\varepsilon > 0$, we can design a fixed ANN architecture, such that
for every $f\in W^s([0,1]^d)$ we can choose the weights of this ANN in such a way that 
it approximates $f$ in $L_2([0,1]^d)$ up to the error $\varepsilon\|f\|_{W^s}$.

The use of the Riesz basis $\mathcal{R}_d$ from \eqref{eq:intro_Riesz_d} allows us to avoid decompositions of unity and Taylor's theorem, which were used frequently
in the analysis of approximation properties of neural networks (see, e.g., \cite[Theorem 1]{Y17} or \cite{S23}). This
allows to keep track of the dependence of the approximation error on the dimension.
Moreover, we achieve a (nearly) optimal number of layers, their widths and the number of weights used. 
This is shown by comparison with available upper and lower bounds.

We shall rely on an upper bound of the number of integer lattice points in a $d$-dimensional
ball of radius $t$ centered in the origin. The exact behavior of this quantity represents a classical problem
in number theory, known as the \emph{Gauss circle problem}.

For an integer $d\ge 1$ and a real number $t\ge 0$, we put
\[
Z(t,d):=\{k\in\Z^d:\|k\|_2\le t\}=\Z^d\cap tB_2^d,
\]
where $B_2^d$ is the unit ball in $\R^d$ and $tB_2^d$ is its $t$-multiple. Furthermore, we denote
\begin{equation}\label{def:N(t,d)}
N(t,d):=\# Z(t,d).
\end{equation}
Much is known in the asymptotic regime where $d$ is fixed and $t$ tends to infinity.
For example, if $d=2$, then $N(t,2)$ is approximated by 
the area of the circle  $\pi t^2$ 
\[
N(t,2)=\pi t^2+E(t)
\]
for some error term $E(t)$.  Gau\ss{} managed to prove $|E(t)|\leq 2\sqrt{2}\pi t$ in this context.
It is conjectured that $|E(t)|={O}(t^{1/2+\varepsilon})$ for every $\varepsilon>0$.
On the other hand, it was established independently by Landau and Hardy that the statement fails for $\varepsilon=0$. 

Here, we are interested in an upper bound, which does not have to be so sharp, but which is valid also in the non-asymptotic regime,
i.e., which holds for all $d\in\N$ and $t\ge 0$.
We provide such a bound in Lemma \ref{lem:lattice_d} below.
Let us remark, that one could obtain a weaker bound by exploiting the well-known results for  entropy numbers
of embeddings of finite-dimensional sequence spaces, see \cite{KMU} for a similar approach.

\medskip

\begin{minipage}{0.45\textwidth}
Let us note that $$N(t,d)=1$$ for $0\le t<1$ and $$N(t,d)=2d+1$$ for $1\le t<\sqrt{2}.$ 
\end{minipage}\hfill
\begin{minipage}{0.45\textwidth}
\begin{center}\includegraphics[width=6cm]{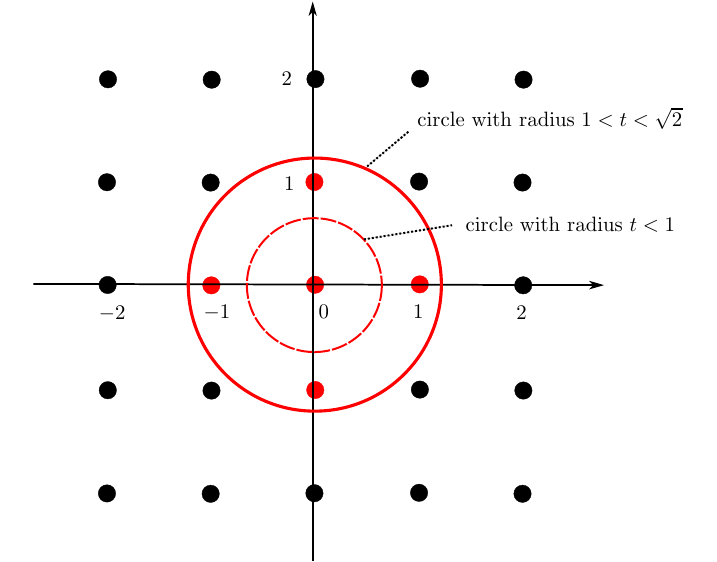}
\captionof{figure}{Illustration of $N(t,2)$}
\end{center}
\end{minipage}


\medskip

\begin{lem}\label{lem:lattice_d} There exist two absolute constants $c_1, c_2>0$
such that for every $d\ge 1$ and $t\ge 0$ the following holds.\\
(i) If $t\ge\sqrt{d}/2$, then 
\[
N(t,d)\le \biggl(\frac{c_1 t}{\sqrt{d}}\biggr)^d.
\]
(ii) If $0<t\le \sqrt{d}/2$, then
\[
N(t,d)\le \biggl(\frac{c_2 d}{t^2}\biggr)^{t^2}.
\]
\end{lem}

\begin{proof}If $t\ge \sqrt{d}/2$, then 
\begin{equation}\label{help}
Z(t,d)+[0,1]^d \subset (t+\sqrt{d}) B_2^d,
\end{equation}
see Figure~\ref{fig:split},
and
\begin{align*}
N(t,d)&=\vol(Z(t,d)+[0,1]^d)
\le (t+\sqrt{d})^d\vol(B_2^d)\\
&= (t+\sqrt{d})^d\frac{\pi^{d/2}}{\Gamma(\frac d2+1)}
\le \biggl(\frac{c_1 t}{\sqrt{d}}\biggr)^d.
\end{align*}
Using Stirling's formula 
it can be checked that $c_1:=3\sqrt{2\pi e}$ works.

If $t\le \sqrt{d}/2$, the proof is more delicate. Recall that, if $0< t<1$, then $N(t,d)=1$
and the result follows. If $1\le t<\sqrt{2}$, then $N(t,d)=2d+1$ and the upper bound again follows. This covers the cases $1\le d\le 7$.

Next, we consider the case when $d\ge 8$ and $\sqrt{d-1}/2\le t\le \sqrt{d}/2$. Then, by (i),
\[
N(t,d)\le N(\sqrt{d}/2,d)\le (c_1/2)^d\le (c_2 d/t^2)^{t^2},
\]
where we used also that $4\le d/t^2\le 4\cdot\frac{d}{d-1}\le 5.$ 

If $d\ge 8$ and $\sqrt{2}\le t\le \sqrt{d-1}/2$, we proceed by induction. 

\begin{figure}[h!]
\centering
\includegraphics[width=0.58\linewidth]{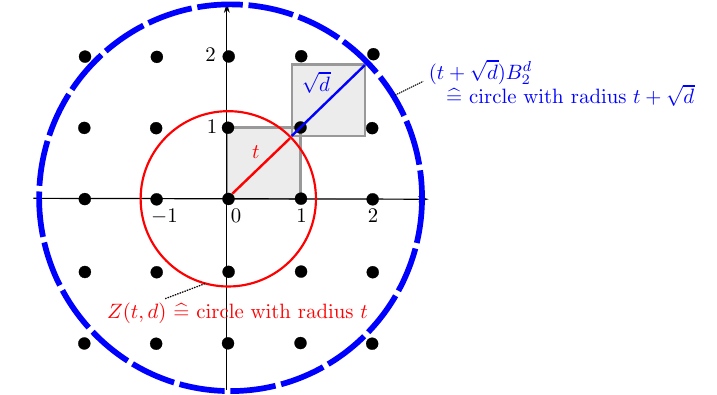} 
\hspace*{-5mm}
\includegraphics[width=0.4\linewidth]{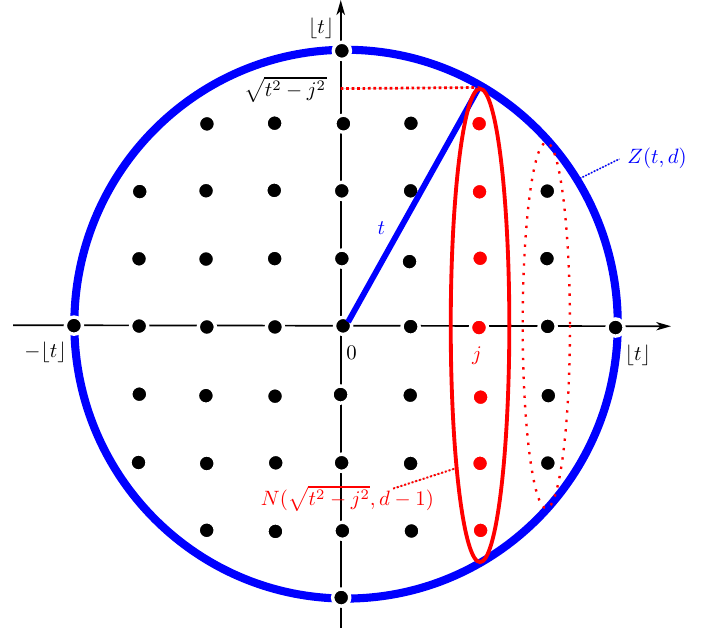}
\caption{Visualization of \eqref{help}, and splitting of $Z(t,2)$ into disks with 
$|k_1|\le\lfloor t\rfloor$}
\label{fig:split}
\end{figure}

We split
\[
Z(t,d)
=\{k\in\Z^d:\|k\|_2\le t\}\\
=
\bigcup_{j=-\lfloor t\rfloor}^{j=\lfloor t\rfloor}\{k\in\Z^d: \|k\|_2\le t, \  k_1=j \}, 
\]
see Figure~\ref{fig:split},
%
and obtain

\begin{align}
\label{eq:N_est} N(t,d)&=\sum_{j=-\lfloor t\rfloor}^{j=\lfloor t\rfloor} N(\sqrt{t^2-j^2},d-1)
=N(t,d-1)+2\sum_{j=1}^{j=\lfloor t\rfloor}N(\sqrt{t^2-j^2},d-1)\\
\notag&\le N(t,d-1)+2\sum_{j=1}^{j=\lfloor t\rfloor}\biggl(\frac{c_2(d-1)}{t^2-j^2}\biggr)^{t^2-j^2}
\le \biggl(\frac{c_2(d-1)}{t^2}\biggr)^{t^2}+4 \biggl(\frac{c_2(d-1)}{t^2-1}\biggr)^{t^2-1},
\end{align}
where we used inductively (ii) for $d-1$ combined with the fact that $t\le \sqrt{d-1}/2$. Furthermore, to estimate the sum, we used
that every term in this sum is at least twice as large as the next term (for $c_2$ large enough).
Indeed, the ratio of the $(j+1)$-st and the $j$-th term in \eqref{eq:N_est} is equal to
\begin{equation}\label{eq:N_est_2}
\left[\frac{t^2-j^2}{c_2(d-1)}\right]^{2j+1}\left(1+\frac{1}{\lambda_j}\right)^{\lambda_j(2j+1)}, \quad \text{where}\quad
\lambda_j=\frac{t^2-(j+1)^2}{2j+1}.
\end{equation}
Using the elementary inequality $(1+1/z)^z\le e$ valid for $z>0$, we observe that \eqref{eq:N_est_2} is smaller than 1/2
for $c_2>0$ large enough. Therefore, taking the largest term with $j=1$ out of the sum, we can bound the sum by
a geometric series with the sum $\frac{1}{1-\frac 12}$, which in the end gives the 
additional factor $2$ in the  last step in \eqref{eq:N_est}.

Finally, for $c_2\ge 4e$, we apply the mean value theorem to the function $f(u)=u^{t^2}$ and obtain
\begin{align*}
4 \biggl(\frac{c_2(d-1)}{t^2-1}\biggr)^{t^2-1}&=\frac{4}{c_2}\cdot \left(1+\frac{1}{t^2-1}\right)^{t^2-1}\cdot
\left(\frac{c_2}{t^2}\right)^{t^2}\cdot t^2(d-1)^{t^2-1}\\
&\le \frac{4e}{c_2}\cdot \left(\frac{c_2}{t^2}\right)^{t^2}[d^{t^2}-(d-1)^{t^2}]\le
\left(\frac{c_2d}{t^2}\right)^{t^2}-\left(\frac{c_2(d-1)}{t^2}\right)^{t^2}.
\end{align*}
Together with \eqref{eq:N_est} we conclude that
\[
N(t,d)\le \biggl(\frac{c_2d}{t^2}\biggr)^{t^2}.
\]
\end{proof}

The last combinatorial lemma directly leads to our first approximation result using ANNs. Note that the architecture 
and the actual ANN 
mentioned in the theorem is given explicitly, 
by means of partial sums of~\eqref{eq:f_Riesz_d}, see Lemma~\ref{lem:ANN}.

\begin{thm}\label{thm:ANN_Ws}
Let $0<s<1$. Then there is a constant $C_s>0$, depending only on $s$, such that for every $d\in\N$ and $0<\varepsilon<1$ the following statement is true. \\
Let $R:=(C_s/\varepsilon)^{1/s}$ and $N(R,d)$ be as in Lemma~\ref{lem:lattice_d}. 
Then, for every $f\in W^s([0,1]^d)$ there is an ANN $\cN\in\Upsilon_d^{W,L}$  with 
\[
W=4\cdot N(R,d)\quad\text{and}\quad L \le 4+\log_2\left(R\cdot\sqrt{\min(R,d)}\right)
\]
such that 
\[
\|f-\cN\|_2 \,\le\, \eps\cdot\|f\|_{W^s}. 
\]
Moreover, 
$\cN$ can be given explicitly depending on $\alpha_k,\beta_k$ with $\|k\|_2\le R$ in~\eqref{eq:f_Riesz_d}
and the architecture of $\cN$ is independent of $f$.
\end{thm}

\begin{proof}
Let $f\in W^s([0,1]^d)$ with $\|f\|_{W^s([0,1]^d)}\le 1$. By Theorem \ref{thm:FsWs_d}, 
$f\in {\mathcal F}^s([0,1]^d)$ with $\|f\|_{{\mathcal F}^s}\le c$ with $c$ independent of $d$.
Therefore, we can decompose $f$ as in \eqref{eq:f_Riesz_d} with
$\alpha_0^2+\sum_{k\, \righttriangleplus\, 0}\|k\|^{2s}_2(\alpha_k^2+\beta_k^2)\le c^2$.

For $R>0$ we decompose
\[
f=f_R+f^R = \Bigl(\alpha_0 + \sum_{k\,\righttriangleplus\,0: \|k\|_2\le R}\dots\Bigr)+
\Bigl(\sum_{k\,\righttriangleplus\,0: \|k\|_2>R}\dots\Bigr).
\]

We reconstruct $f_R$ exactly as an ANN, the error is given by $f^R$:
\begin{equation*}
\|f-f_R\|^2_2\,=\,\|f^R\|_2^2 
\,\le\, c'\,\sum_{k\,\righttriangleplus\,0: \|k\|_2>R} (\alpha_k^2+\beta_k^2)\cdot \frac{\|k\|_2^{2s}}{R^{2s}}\le C^2 R^{-2s}.
\end{equation*}
To make this smaller than $\varepsilon^2$, we choose $R:=(C/\varepsilon)^{1/s}.$

By Lemma~\ref{lem:ANN} with $I:=Z(R,d)$, we see that $f_R\in \Upsilon_d^{W,L}$ with
\[
W=4 N(R,d)\quad\text{and}\quad 
L=4+\max_{k\in Z(R,d)}\log_2(\|k\|_1) \le 4+ \log_2\left(R\cdot\sqrt{\min(R,d)}\right).
\]
In the last estimate we used H\"older's inequality together with $\|k\|_2\leq R$ and the simple observation that the number of non-zero indices
of $k\in Z(R,d)$ is bounded from above by $\min(R, d)$. 
\end{proof}

\medskip

\begin{rem}
As the proof above is based on standard estimates from linear approximation, 
the result can easily be generalized to approximation in other norms, like the uniform norm. 
However, this might lead to additional assumptions and $d$-dependent factors. 
\end{rem}

\medskip

We now combine Theorem~\ref{thm:ANN_Ws} with Lemma~\ref{lem:lattice_d} to bound the parameters $W$ and $L$ of
the constructed neural network by quantities involving only $\varepsilon$ and $d$. 

\begin{cor} \label{cor:ANN_Ws}
Let $0<s<1$, $d\in\N$, and $0<\varepsilon<1$. There are constants $c_3,c_4, c_5$ that depend only on $s$ such that 
\[
\inf_{\cN\in\Upsilon_d^{W,L}}\|f-\cN\|_2 \,\le\, \eps\cdot\|f\|_{W^s} 
\]
for every $f\in W^s([0,1]^d)$ 
if 
\[
\varepsilon\ge c_3\, d^{-s/2}, \quad 
W \ge (c_4 \eps^{2/s} d)^{c_5\, \eps^{-2/s}}, 
\quad\text{and}\quad 
L\ge c_4\,\log_2(1/\varepsilon),
\]
or 
\[
\varepsilon\le c_3\, d^{-s/2}, \quad 
W \ge \left(\frac{c_4}{\eps^{1/s} \sqrt{d}}\right)^{d}, 
\quad\text{and}\quad 
L\ge c_4\,\log_2(1/\varepsilon). 
\]
\end{cor}

\subsection{Barron classes}

When dealing with the Barron classes ${\mathbb B}^s([0,1]^d)$
we first need to select the most important terms from \eqref{eq:f_Riesz_d}, which get reconstructed by the neural network.
This is captured by the concept of best $n$-term approximation, a concept,
which seems to go back as far as to \cite{Schmidt} and which is frequently used in nonlinear approximation theory \cite{DeVore}.

Let $X\subset Y$ 
be two Banach spaces and let $\Phi\subset X$ be a set of elements of $X$.
For $f\in X$ and a positive integer $n\ge 1$, we denote
\[
\sigma_n(f):=\sigma_n(f,Y,\Phi)=\inf\left\{\left\|f-\sum_{j=1}^n \alpha_j\varphi_j\right\|_Y:\alpha_1,\dots,\alpha_n\in\C, \varphi_1,\dots,\varphi_n\in \Phi\right\}
\]

and
\[
\sigma_n(X,Y,\Phi):=\sup_{\|f\|_X\le 1}\sigma_n(f).
\]
Here, we only consider the case where $X$ and $Y$ are Banach spaces of sequences, 
in which case we always choose $\Phi:=\{e_n\}$ the set of canonical sequences with $(e_n)_j=1$ if $n=j$ and $(e_n)_j=0$ otherwise,
and we write $\sigma_n(X,Y)$ for $\sigma_n(X,Y,\Phi)$.



\medskip

We start with the sequence spaces corresponding to univariate Barron classes. 

\begin{lem}
Let $s\ge 0$ and let $b_s^1$ be the space of bounded sequences $\alpha=(\alpha_k)_{k=1}^\infty$, for which the norm
\[
\|\alpha\|_{b_s^1}:=\sum_{k=1}^\infty k^{s} |\alpha_k|
\]
is finite. 
Then, 
$\sigma_n(b_s^1,\ell_2)\le c\, n^{-s-1/2}$ 
for every $n\in\N$  
where $c>0$ only depends on $s$.
\end{lem}
\begin{proof} First, note that $b_s^1\hookrightarrow \ell_1\hookrightarrow \ell_2$ and that 
$\sigma_n(b_s^1,\ell_2)$ 
is well-defined.
We shall use a classical result on rearrangements of sequences, see~\cite[Section 10.2, Theorem 368]{HLP},
which states that for two non-negative sequences $a=(a_j)_{j=1}^\ell$ and $b=(b_j)_{j=1}^\ell$
with non-increasing rearrangements $(a^*_j)_{j=1}^\ell$ and $(b^*_j)_{j=1}^\ell$, we have that 
\[
\sum_{j=1}^\ell a^*_j b^*_{\ell-j+1}
\le \sum_{j=1}^\ell a_j b_j\le \sum_{j=1}^\ell a^*_j b^*_j.
\]
Actually, using the same proof as in \cite{HLP}, one can show that for every non-negative sequence $a=(a_j)_{j=1}^\infty$ converging to zero
and for every non-negative non-decreasing sequence  $b=(b_j)_{j=1}^\infty$ it holds that
\begin{equation}\label{eq:HLP}
\sum_{j=1}^\infty a_j^* b_j\le \sum_{j=1}^\infty a_jb_j.
\end{equation}

Fix now $\alpha\in b_s^1$ with $\|\alpha\|_{b_s^1}\le 1$. Then we can apply \eqref{eq:HLP} to $a_j=|\alpha_j|$ and $b_j=j^s$ and obtain
\[
|\alpha_\ell|^* \sum_{k=1}^\ell k^{s}\le \sum_{k=1}^\ell k^{s} a_k^* \le \sum_{k=1}^\infty k^{s} a_k^* \le \sum_{k=1}^\infty k^{s} |\alpha_k|\le 1.
\]
This gives
$|\alpha_\ell|^*\le (s+1)\cdot \ell^{-s-1}$.
The best $n$-term approximation of $\alpha$ in $\ell_2$
is given by the $n$ largest coefficients of $\alpha$. Therefore we get
\[
\sigma_n(\alpha)^2=\sum_{k=n+1}^\infty (|\alpha_k|^{\ast})^2 \le (s+1)^2 \sum_{k=n+1}^\infty k^{-2s-2} \le \frac{(s+1)^2}{2s+1} n^{-2s-1}.
\]
\end{proof}
The generalization to sequence spaces corresponding to Barron classes of multivariate functions follows the same pattern, but is more technical.
Let $d\ge 1$ and $s\ge 0$.  Then $b_s^d$ is the space of bounded  sequences $\alpha=(\alpha_k)_{k\,\righttriangleplus\,0}$ for which the norm
\[
\|\alpha\|_{b_s^d}:=\sum_{k\,\righttriangleplus\,0} \|k\|_2^s\cdot |\alpha_k|
\]
is finite.

\begin{lem}\label{lem:sigma_n}
Let $d\ge 2$ and $s\ge 0$. Then, 
there are absolute constants $c_1,c_2>0$, 
and some $C_s>0$ that depends only on $s$, such that 
\begin{equation}\label{eq:sigman}
    \sigma_n(b_s^d,\ell_2)^2
    \,\le\, C_s\,\cdot\begin{cases}
    n^{-1},&\ \text{if}\ 1\le n\le c_2d,\\
    \Bigl(\frac{\log(c_2d)}{\log(n)}\Bigr)^s\cdot n^{-1},&\ \text{if}\ c_2d\le n \le (c_1/2)^d,\\
    d^{-s}\cdot n^{-2s/d-1},&\ \text{if}\ (c_1/2)^d\le n.
    \end{cases}
\end{equation}
\end{lem}

\begin{proof} \emph{Step 1.}\\
Let $\alpha\in b_s^d$ with $\|\alpha\|_{b_s^d}\le 1$.
We define the weight sequence $\omega=(\omega_k)_{k\,\righttriangleplus\,0}$ with $\omega_k=\|k\|_2^{s}$.
We denote by $\alpha^*=(\alpha_\ell^*)_{\ell=1}^\infty$ the one-dimensional non-increasing rearrangement of $(|\alpha_k|)_{k\righttriangleplus 0}$ and by $\omega^{\#}=(\omega^{\#}_\ell)_{\ell=1}^\infty$
the one-dimensional non-decreasing rearrangement of $\omega$.
Again, we have
\[
\alpha_\ell^*\sum_{k=1}^\ell \omega_k^{\#}\le \sum_{k=1}^\ell \alpha^*_k \omega^{\#}_k\le
\sum_{k=1}^\infty \alpha^*_k \omega^{\#}_k\le \sum_{k\,\righttriangleplus\,0}\omega_k |\alpha_k|\le 1,
\]
i.e. $\alpha_\ell^*\le W(\ell,s,d)^{-1}$, where $W(\ell,s,d)=\sum_{k=1}^\ell \omega_k^{\#}$.
The best $n$-term approximation of $\alpha$ in $\ell_2$ is then again given by
\begin{equation}\label{eq:sigma_est}
\sigma_n(\alpha)^2=\sum_{k=n+1}^\infty (\alpha_k^*)^2\le \sum_{k=n+1}^\infty W(k,s,d)^{-2}.
\end{equation}
\emph{Step 2.} Before we prove a lower bound for $W(k,s,d)$, we first establish lower bounds for  the terms of $\omega^{\#}$.
This we do by exploiting Lemma~\ref{lem:lattice_d}. Indeed, if $N(t,d)<n$, then also
\[
\#\{k\in\Z^d: k\,\righttriangleplus\,0\ \text{and}\ \|k\|_2\le t\}\le N(t,d)<n.
\]
Hence, the lattice point with $n$-th smallest value of $\omega$ lies outside of $tB_2^d$ and, therefore,
\[
\omega_n^{\#}> t^{s}.
\]
Combining this observation with Lemma~\ref{lem:lattice_d}, we will prove that
\begin{align}\label{eq:omega_low}
    \omega_n^{\#}\ge \begin{cases}
    1,&\text{if}\ 1\le n\le c_2d,\\
    \displaystyle\left(\frac{\log(n)}{\log(c_2d)}\right)^{s/2},&\text{if}\ c_2d<n\le (c_1/2)^d,\\
    \displaystyle\left(\frac{n^{1/d}\sqrt{d}}{c_1}\right)^s, &\text{if}\ (c_1/2)^d<n
        \end{cases}    
\end{align}
with $c_1,c_2>0$ from Lemma~\ref{lem:lattice_d}, 
with possibly larger $c_2$ to ensure $c_2\ge (c_1/2)^4$.

Indeed, if $1\le n\le c_2d$, then \eqref{eq:omega_low} follows easily. If $n>(c_1/2)^d$, then we choose any $t$ with $\sqrt{d}/2\le t<n^{1/d}\sqrt{d}/c_1$.
Then, by Lemma \ref{lem:lattice_d}, $N(t,d)\le (c_1t/\sqrt{d})^d<n$ and $\omega_n^{\#}> t^s$.
As $t$ was arbitrary within these limits, we get $\omega_n^{\#}\ge  (n^{1/d}\sqrt{d}/c_1)^s$.

Finally, if $c_2d< n\le (c_1/2)^d$, then we put $t^2=\log(n)/\log(c_2d)>1$ and obtain
\[
t^2\le 
\frac{d\log(c_1/2)}{\log(c_2)}\le \frac{d}{4},
\]
where we used that $c_2\ge (c_1/2)^4$. Therefore, Lemma \ref{lem:lattice_d}, gives
\[
\log(N(t,d))\le t^2\log(c_2 d/t^2)<t^2 \log(c_2 d)=\log(n)
\]
and $\omega_n^{\#}> t^s=(\log(n)/\log(c_2d))^{s/2}$, which finishes the proof of \eqref{eq:omega_low}.

\emph{Step 3.}  
Next, we show that
\begin{equation}\label{eq:W_lower}
W(\ell,s,d)\ge \begin{cases}
    \ell,&\text{if}\ 1\le\ell\le c_2d,\\
    (1-s/2)\ell\left(\frac{\log(\ell)}{\log(c_2d)}\right)^{s/2},&\text{if}\ c_2d<\ell\le (c_1/2)^d,\\
    c\,d^{s/2}\cdot\ell^{s/d+1},&\text{if}\ (c_1/2)^d\le \ell,
\end{cases}
\end{equation}
where $c>0$ depends only on $s$.

If $1\le \ell \le c_2d$, this follows immediately from \eqref{eq:omega_low}. If $c_2 d\le \ell\le (c_1/2)^d$,
then we estimate
\begin{align*}
W(\ell,s,d)&=\sum_{k=1}^{c_2d}\omega_k^{\#}+\sum_{k=c_2d+1}^{\ell}\omega_k^{\#}
\ge c_2d +\frac{1}{\log(c_2d)^{s/2}}\sum_{k=c_2d+1}^\ell (\log k)^{s/2}\\
&\ge c_2d +\frac{1}{\log(c_2d)^{s/2}}\int_{c_2d}^\ell (\log t)^{s/2}dt\\
&=c_2d + \frac{\ell\log(\ell)^{s/2}-c_2d\log(c_2d)^{s/2}}{\log(c_2d)^{s/2}}-\frac{s}{2}{\frac{1}{{\log(c_2d)^{s/2}}}}\int_{c_2d}^\ell(\log t)^{s/2-1}dt\\
&\ge \ell\cdot \frac{\log(\ell)^{s/2}}{\log(c_2d)^{s/2}}-\ell\cdot \frac{s}{2}{\frac{\log(\ell)^{s/2}}{{\log(c_2d)^{s/2}}}},
\end{align*}
which gives the second estimate in \eqref{eq:W_lower}.

It remains to prove \eqref{eq:W_lower} for $\ell\ge (c_1/2)^d$. To simplify the presentation, we assume that $c_1\ge 4$ is an even number.
We start with $\ell=(c_1/2)^d$. We set $t^2=\gamma d$ for $\gamma>0$ small enough to ensure that 
\[
N(t,d)\le (c_2d/t^2)^{t^2}=(c_2/\gamma)^{\gamma d}<(c_1/2)^d/2=\ell/2.
\]
This ensures that $\omega^{\#}_{\ell/2}>t^s=(\gamma d)^{s/2}$ and also
\[
W(\ell,s,d)\ge \sum_{k=\ell/2}^{\ell} \omega_k^{\#}\ge \frac{\ell}{2}\omega^{\#}_{\ell/2}\ge \frac{\ell}{2}(\gamma d)^{s/2}
=c\, d^{s/2}\ell^{s/d+1},
\]
where $c$ depends only on $s$ (and the absolute constants $c_1$ and $c_2$).

If $(c_1/2)^d\le \ell \le 4(c_1/2)^d$, the proof of \eqref{eq:W_lower} follows by monotonicity
\[
W(\ell,s,d)\ge W((c_1/2)^d,s,d)\ge c d^{s/2} [4(c_1/2)^d]^{s/d+1}\cdot 4^{-2}\ge c' d^{s/2}\ell^{s/d+1}.
\]

Finally, for $\ell \ge 4(c_1/2)^d$, we use \eqref{eq:omega_low} and estimate
\[
W(\ell,s,d)\ge \sum_{k=(c_1/2)^d+1}^\ell \omega_k^{\#}\ge c_1^{-s} d^{s/2} \sum_{k=(c_1/2)^d+1}^\ell k^{s/d},
\]
which finishes the proof of \eqref{eq:W_lower}.

\emph{Step 4.} As the last step, we prove \eqref{eq:sigman}.

If $n\ge (c_1/2)^d$, we use \eqref{eq:sigma_est} and obtain
\begin{align*}
\sigma_n(\alpha)^2\le \sum_{k=n+1}^\infty W(k,s,d)^{-2}\le \sum_{k=n+1}^\infty c^{-2} d^{-s} k^{-2(s/d+1)}.
\end{align*}

If $c_2d\le n \le (c_1/2)^d$, then we estimate similarly
\begin{align}
\notag \sigma_n(\alpha)^2&=\sum_{k=n+1}^{(c_1/2)^d}W(k,s,d)^{-2}+\sum_{k=(c_1/2)^d+1}^\infty W(k,s,d)^{-2}\\
\label{eq:sigma_middle}&\le c\, \log(c_2d)^s \sum_{k=n+1}^{(c_1/2)^d} k^{-2}\log(k)^{-s}+c\, d^{-s}(c_1/2)^{-2s-d}\\
\notag &\le c\, \Bigl(\frac{\log(c_2d)}{\log(n)}\Bigr)^s\cdot n^{-1}.
\end{align}

And finally, if $1\le n\le c_2d$ we combine \eqref{eq:sigma_est} and \eqref{eq:sigma_middle} to 
\begin{align*}
\sigma_n(\alpha)^2&\le \sum_{k=n+1}^{c_2d} W(k,s,d)^{-2} + c\, (c_2d)^{-1}\le \sum_{k=n+1}^{c_2d} k^{-2} + c\, (c_2d)^{-1}
\end{align*}
and the result follows. \\
\end{proof}

We now show the analogue of our ANN approximation result in Theorem ~\ref{thm:ANN_Ws} for Barron classes.

\begin{thm}\label{thm:ANN_Bs}
Let $0<s<1$. Then there are constants $c,C>0$, depending only on $s$, such that
for every $d\in\N$ and $0<\varepsilon<1$ the following statement is true. \\
Let $R:=(C/\varepsilon)^{1/s}$ and 
$n$ such that $\sigma_n(b_s^d,\ell_2)<c\eps$. 
Then, for every $f\in \mathbb{B}^s([0,1]^d)$
there is an ANN $\cN\in\Upsilon_d^{W,L}$ with 
\[
W=4 n\quad\text{and}\quad L
\le 4+\log_2\left(R\cdot\sqrt{\min(R,d)}\right)
\]
such that 
\[
\|f-\cN\|_2 \,\le\, \eps\cdot\|f\|_{\mathbb{B}^s}. 
\]
Moreover, 
$\cN$ can be given explicitly 
depending on 
the coefficients of the best $n$-term approximation of $f$.
\end{thm}

\begin{proof}
Let $f$ be the unit ball of ${\mathbb B}^s([0,1]^d)$. By Theorem \ref{thm:Barron}, there is a constant $C_s>0$, such that
$\|f\|_{{\mathcal B}^s([0,1]^d)}\le C_s$. We can therefore decompose $f$ as in \eqref{eq:f_Riesz_d} into a series
\[
f=\alpha_0+\sum_{k\righttriangleplus 0}\Bigl[\alpha_k\CalC_k+\beta_k\CalS_k\Bigr]
\]
with
\[
\|\alpha\|_{b^d_s}=\sum_{k\righttriangleplus 0} \|k\|_2^s \cdot|\alpha_k|\le C_s \quad\text{and}\quad
\|\beta\|_{b^d_s}=\sum_{k\righttriangleplus 0} \|k\|_2^s \cdot|\beta_k|\le C_s.
\]
We define 
$R:=(C/\varepsilon)^{1/s}$ as well as
$\alpha^R=(\alpha^R_k)_{k\righttriangleplus 0}$ by $\alpha^R_k=\alpha_k$ if $\|k\|_2\le R$ and $\alpha^R_k=0$ otherwise.
Then
\begin{align} \label{eq:ANN-Bs-calc}
\|\alpha-\alpha^R\|_2&=\biggl(\sum_{k\righttriangleplus 0: \|k\|_2>R}|\alpha_k|^2\biggr)^{1/2}
\le \sum_{k\righttriangleplus 0: \|k\|_2>R}|\alpha_k|\\
\notag&\le R^{-s}\sum_{k\righttriangleplus 0: \|k\|_2>R}\|k\|_2^s\cdot |\alpha_k|\le C_s R^{-s}.
\end{align}
Next, we define $\widehat \alpha^{R,n}$ to be the best $n$-term approximation of $\alpha^R$ in $\ell_2$ (with respect to the
canonical basis of $\ell_2$).
Hence,
\[
\|\alpha^R-\widehat \alpha^{R,n}\|_2\le \|\alpha^R\|_{b^d_s}\cdot \sigma_n(b_s^d,\ell_2)\le C_s\cdot \sigma_n(b_s^d,\ell_2).
\]
Furthermore, we define $\beta^R$ and $\widehat \beta^{R,n}$ similarly.

We approximate $f$ by $f^{R,n}$, which is defined as follows
\[
f^{R,n}=\alpha_0+\sum_{k\righttriangleplus 0}\Bigl[\widehat\alpha^{R,n}_k\CalC_k+\widehat\beta^{R,n}_k\CalS_k\Bigr]
\]
Then, 
\begin{align*}
\|f-f^{R,n}\|_2&=\left\|\sum_{k\righttriangleplus 0}\Bigl[(\alpha_k-\widehat\alpha^{R,n}_k)\CalC_k+(\beta_k-\widehat\beta^{R,n}_k)\CalS_k\Bigr]\right\|_2\\
&\le \frac{\sqrt{2}}{2}\cdot \Bigl(\|\alpha-\widehat\alpha^{R,n}\|_2 + \|\beta-\widehat\beta^{R,n}\|_2\Bigr)\qquad \\
&\le \frac{\sqrt{2}}{2}\cdot \Bigl(\|\alpha-\alpha^{R}\|_2+\|\alpha^R-\widehat\alpha^{R,n}\|_2 + \|\beta-\beta^{R}\|_2+\|\beta^R-\widehat\beta^{R,n}\|_2\Bigr)\\
&\le C_s'\cdot [R^{-s}+\sigma_n(b_s^d,\ell_2)]\le\varepsilon, 
\end{align*}
where we use the choice of $n$.

Finally, by 
Lemma~\ref{lem:ANN}, 
$f^{R,n}$ can be reproduced by an artificial neural network $f^{R,n}\in \Upsilon_d^{W,L}$ with
\[
W=4n\quad \text{and} \quad 
L\le 4+\log_2\left(R\cdot\sqrt{\min(R,d)}\right)
\]
where we again use that $\|k\|_1\le\sqrt{\min(R,d)}\|k\|_2$ for $k\in\Z^d$. \\
\end{proof}

\section{Discussion} \label{sec:discussion}

In the last section, we studied the approximation of functions from Sobolev and Barron classes using different ANNs. 
Let us finally highlight and discuss a few interesting cases and compare them with the relevant literature. 

\subsection{Sobolev spaces}

The potential of shallow and deep neural networks in reproducing functions from classical function spaces
has been studied intensively for several decades in different regimes (cf. \cite{
MLP16, MP16, Pinkus99, T, T2}).
The influential paper \cite{Y17} studied a setting very similar to ours,
although the function spaces (and the error of approximation) involved the uniform norm instead of the $L_2$-norm used in our work.
To be more specific, for positive integers $n,d$, let $F_{n,d}$ denote a set of functions defined on $[0,1]^d$, which have all partial derivatives
up to the order $n$ uniformly bounded by one. Then \cite{Y17} shows that
\begin{itemize}
\item there is a fixed architecture of an ANN, which (by a suitable choice of weights and biases)
can approximate every function $f\in F_{n,d}$ uniformly up to error $0<\varepsilon<1$ and which has depth $\mathcal{O}(\log_2(1/\varepsilon)+1)$
and $\mathcal{O}(\varepsilon^{-d/n}(\log(1/\varepsilon)+1))$ weights;
\item choosing the network architecture adaptively depending on $f$ can lead to a lower number of weights;
\item if a network architecture can approximate every $f\in F_{n,d}$ uniformly up to   error $\varepsilon$, then it must have at least $c\,\varepsilon^{-d/(2n)}$
weights.
\end{itemize}
We stress that essentially all implicit constants in \cite{Y17} depend both on $n$ and $d$.
From this point of view, we prove a variant of \cite[Theorem~1]{Y17} for Sobolev classes built upon $L_2$
and approximation in $L_2$ and, in contrast to \cite{Y17}, we achieve explicit $d$-dependence.

\medskip

Indeed, if we fix $d\ge 1$ to be a constant and let $\eps\to0$,
then Theorem \ref{thm:ANN_Ws} (see Corollary~\ref{cor:ANN_Ws}) provides a fixed architecture, which allows to approximate every $f$
from the unit ball of $W^s([0,1]^d)$ up to error $\varepsilon$ (in the $L_2([0,1]^d)$ norm) with
\[
W\sim_{s,d} \eps^{-d/s}\quad \text{and} \quad L\sim_{s,d} \log(1/\eps).
\]
The corresponding ANN has ${\mathcal O}(\varepsilon^{-d/s}\max(\log(1/\varepsilon),d))$
nonzero parameters (weights and biases), 
see the discussion after Lemma~\ref{lem:ANN}.
This rate is optimal due to the continuous dependence of the parameters of the ANN on $f$, see \cite{DHM} and \cite[Theorem~1]{Y17}.

\bigskip

Another line of research that comes close to our work is represented, for example, by \cite{DeVore-2,GKP,50,S23}. To compare our results with \cite{S23}, we note that 
Theorem \ref{thm:ANN_Ws} shows, that for $f\in W^s([0,1]^d)$, $0<s<1$, it holds 
\[
\inf_{f_R\in \Upsilon_d^{W,L}}\|f-f_R\|_{L_2([0,1]^d)}\leq C\|f\|_{W^s([0,1]^d)}R^{-s}, 
\]
where the appearing constant $C$ is independent of $d$ and width and depth are chosen as in Theorem \ref{thm:ANN_Ws}, i.e.,   $W=4N(R,d)$, $L=4+ \log_2(R\cdot\sqrt{\min(R,d)})$.
Without much additional work (cf. Lemma~\ref{lem:ANN}) we can obtain the same result also with
\begin{equation}\label{eq:siegl_us}
W=d+3\quad \text{and}\quad L\le 2N(R,d)\log_2(16\cdot R\sqrt{\min(R,d)}).
\end{equation}
In contrast to this in \cite[Theorem~1]{S23} it was shown that for $0<s<\infty$, $1\leq p,q\leq \infty$, and $W=25d+31$ it holds 
\begin{equation}\label{upper-bound-siegel}
\inf_{f_L\in \Upsilon_d^{W,L}}\|f-f_L\|_{L_p([0,1]^d)}\leq C\|f\|_{W^s_q([0,1]^d)}L^{-2s/d}, 
\end{equation}
for a constant $C:=C(s,r,q,p,d)$. Additionally, \cite[Theorem~3]{S23} also gives a lower bound which shows that the rate above is sharp in terms of the number of parameters. 
For $p=q=\infty$ and $0<s<1$ this corresponds to the setting of \cite{Y17} and for all $s>0$ we refer to \cite{40}. 
Note that  Siegel's results \eqref{upper-bound-siegel}
are more general in terms of the parameters considered in the underlying functions spaces whereas we only cover $p=q=2$ so far.

If we now consider $d$ fixed and let the error of approximation $\varepsilon>0$ go to zero, then the width $W$ considered in \eqref{upper-bound-siegel} is fixed,
but the length $L$ grows to infinity as $L={\mathcal O}(\varepsilon^{-d/(2s)})$.
Note that classical methods of approximation using piecewise polynomials or wavelets can attain an approximation rate of $L^{-s/d}$ with $L$ wavelet coefficients or
piecewise polynomials with $L$ pieces. Therefore, the approximation rate of $CL^{-2s/d}$ is significantly faster than traditional methods of approximation.
This phenomenon has been called the {\em super-convergence} of deep ReLU networks \cite{DD+22,DeVore-2,50,61} and is obtained using a special bit extraction technique \cite{Bartlett},
which gives an optimal encoding of sparse vectors. 

In the setting of $d$ fixed and $\varepsilon$ tending to zero, the width in \eqref{eq:siegl_us} is also fixed
but length grows faster as $L={\mathcal O}(\varepsilon^{-d/s}\cdot\log(1/\varepsilon))$. This is due the fact, that we only consider fixed architectures
and (up to now) do not make use of any variant of the bit-extraction technique. We leave it as an open problem, if such an improvement would be possible
also in our approach.


\medskip

Finally, let us mention in this context  that approximation rates when both the width and depth vary have also been obtained in \cite{50}. There the authors considered H\"older
continuous functions (which again corresponds to $p=q=\infty$ and $s>0$) and proved that for any $N,L\in \mathbb{N}$  it holds that  $\mathrm{ReLU}$ networks with width
$\mathcal{O}(\max\{d\lfloor N^{1/d}\rfloor, N+2\})$ and depth $\mathcal{O}(L)$ can approximate a H\"older function   on $[0,1]^d$ with an approximation rate
$\varepsilon=\mathcal{O}(\lambda \sqrt{d}(N^2L^2\ln N)^{-\alpha/d})$, where $\alpha\in (0,1]$ and $\lambda>0$ are H\"older order and H\"older constant, respectively.
To compare this result with our work, one has to recalculate the dependence of $N$ and/or $L$ on $\varepsilon$, which again reveals the exponential dependence of
the involved constants on the dimension $d$.
We refer also to~\cite{DN21} for approximation in other norms.

\medskip

Finally, our approach allows us to consider also a different regime, namely when $\varepsilon>0$ is fixed and $d$ grows to infinity.
Then (cf. Corollary \ref{cor:ANN_Ws}) the number of layers $L={\mathcal O}(\log_2(1/\varepsilon))$ stays bounded and the width
$W={\mathcal O}(\varepsilon^{2/s}d)^{\gamma \varepsilon^{-2/s}}={\mathcal O}(d^{\gamma \varepsilon^{-2/s}})$
grows polynomially in $d$ and so does the number of all weights  in the network which is of order $\mathcal{O}(W^2L+dW)$, see  \eqref{eq:number}. 
In this sense, we avoid the \emph{curse of dimensionality}.

\subsection{Barron spaces}

The original paper of Barron \cite{Barron} used the Maurey technique \cite{Pisier} (also known as the probabilistic Caratheodory's theorem \cite[Theorem 0.0.2]{Versh})
to show that the functions from the Barron class ${\mathcal B}^1_{\text{ext}}$, cf. \eqref{Barron-orig}, can be approximated by shallow neural networks with $n$ neurons 
up to the precision ${\mathcal O}(n^{-1/2})$ in the $L_2$ sense. This approach leads to a randomized construction of a shallow neural network and its architecture
depends on the approximated function $f$. Finally, it also allowed to give explicit bounds on the constants involved and to show that neural network approximation
of functions from this class avoids the curse of dimensionality. Note also, that \cite{Jones} gives a non-probabilistic proof of Maurey's result.

\medskip

The results of \cite{Barron} were generalized in several directions. Uniform approximation was considered in \cite{Barron3}
and spaces of Barron-type based on the integral representation \eqref{Barron-new} were investigated in \cite{CPV23,EMW22,E1}.
Upper and lower bounds on approximation rates, metric entropy, and $n$-widths of Barron classes were recently obtained in \cite{SXu}.

\medskip

To compare the bound of Theorem \ref{thm:ANN_Bs} with these results, we have to take into account also the best $n$-term approximation bounds of Lemma \ref{lem:sigma_n}.
Naturally, we distinguish several cases.

\medskip

First, we observe that if $C\cdot d^{-1/2}\le \varepsilon \le 1$, then $n={\mathcal O}(\varepsilon^{-2})$ and Theorem \ref{thm:ANN_Bs} provides an ANN with
$W={\mathcal O}(n)={\mathcal O}(\varepsilon^{-2})$ and $L={\mathcal O}(\log_2(1/\varepsilon))$, which approximates given function $f$ from the unit ball of ${\mathbb B}^s([0,1]^d)$
up to $\varepsilon$ precision. The architecture of this ANN depends adaptively on the function $f$, which we want to approximate.
Furthermore, this ANN has ${\mathcal O}(\varepsilon^{-2}\cdot \max(d,\log_2(1/\varepsilon)))$ non-zero weights. All the constants in the ${\mathcal O}$-notation
are independent on $d$ and (up to the $\log$-terms), we indeed recover the results of \cite{Barron}.
Furthermore, it follows that if we fix $1>\varepsilon>0$ constant and let $d$ grow to infinity, then $C\cdot d^{-1/2}\le \varepsilon$ for $d$ large enough
and we indeed avoid the curse of dimensionality.

\medskip

Similarly, if $\varepsilon={\mathcal O}(d^{-s/2}\cdot (c_1/2)^{-s-d/2})$, then the condition $\sigma_n(b_s^d,\ell_2)< c\,\varepsilon$
leads to $n={\mathcal O}(\varepsilon^{-\frac{2d}{2s+d}}\cdot d^{-\frac{sd}{2s+d}})$ (or, equivalently, $\varepsilon={\mathcal O}(d^{-s/2}\cdot n^{-\frac{s}{d}-\frac{1}{2}})$).
This improvement in the asymptotic error decay rate from $1/2$ to $1/2+s/d$ was already observed in \cite{KB} for $s=1$, even for the uniform approximation.

We believe that our approach leads to a more transparent proof, at least for $s<1$. 
It is clearly of interest to extend this technique to higher smoothness and to classes of functions with additional structure.


\section{Approximation with ANNs using function values}
\label{sec:sampling}

We want to study how well one can approximate a function from a class $F$ by ANNs 
if only function values $f(x_i)$ (aka \emph{samples}) for some $x_i$ are known, i.e., 
we consider the \emph{NN-sampling numbers} of $F$ in $G$, which are defined by 
\[
g_n^{{\rm NN}}(F,G,W,L) \,:=\, 
\inf_{\substack{x_1,\dots,x_n\in D\\ \phi\colon \R^n \to \Upsilon_d^{W,L}\\}}\, 
\sup_{f\in F}\, 
\Big\|f - \phi\!\left(f(x_1),\dotsc,f(x_n)\right) \Big\|_{G}, 
\] 
where 
$\Upsilon_d^{W,L}$ is the set of feed-forward neural networks defined on a set $D$ (here: $D=[0,1]^d$) with ReLU activation, 
width $W$ and depth $L$, 
and $G$ is a normed space (here: $G=L_p$) specifying the error measure.
This gives the minimal error achievable with NNs that can be found by any \emph{algorithm} that has only access to $n$ function evaluations of~$f$. 

Note that
we consider the \emph{worst-case error} over a class $F$, 
i.e., we want an algorithm to be ``good'' for all elements of $F$ (which is often the unit ball in a normed space) simultaneously. This accounts for the fact that a specific $f$ is only known through the data, and some assumptions, like a certain regularity. 
A typical \emph{benchmark} in this setting are the \emph{Gelfand numbers}
\[
c_n(F, G) \,:=\,
\inf_{\substack{\psi\colon \C^n\to G\\ 
N\colon F\to\C^n \; \text{ linear}}}\; 
\sup_{f\in F}\, 
\big\|f - \psi\circ N(f) \big\|_{G}, 
\] 
which represent the minimal error of an arbitrary algorithm (without a specified approximation space) that uses $n$ linear measurements.

There has been a lot of interest, 
also recently, 
in finding ANN-approximations based on samples, see e.g.~\cite{Barron2,DD+22,DeVore-2}. 
The used methods are usually tailored to the specific setting and, again, often employ heavy computations and unknown dimension-dependent factors.
The aim of this section is to show that the Riesz basis established above, 
together with 
general results on recovery based on (random) samples, 
gives an easy way to obtain rather explicit bounds. 
For this, let us denote by 
\begin{equation} \label{eq:VR}
V_R:= {\rm span}\Bigl\{\CalC_k,\,\CalS_k
\colon k\in\Z^d,\, k\,{\righttriangleplus}\, 0, \, \|k\|_2\le R  \Bigr\} 
\end{equation}
the finite-dimensional space where we search for the approximation. 
Recall from Lemma~\ref{lem:ANN} (see also the proof of Theorem~\ref{thm:ANN_Ws}) 
that every $g\in V_R$ can be written explicitly as an ANN from $\Upsilon_d^{W,L}$ with suitable $W$ and $L$ 
based on the corresponding coefficients. 
To bound $g_n^{{\rm NN}}$, it is therefore enough to learn a corresponding approximation in $V_R$.

Let us first 
discuss the case of Sobolev spaces $W^s([0,1]^d)$. 
In this case, 
as for more general reproducing kernel Hilbert spaces, 
it has been observed in recent years that a simple least squares approximation 
\begin{equation} \label{eq:alg-ls}
\widehat{f}^{ls}_{V,X}
\,:=\, 
\underset{g\in V}{\rm argmin}\, 
\sum_{i=1}^N 
\vert g(x_i) - f(x_i) \vert^2  
\end{equation} 
onto a suitable subspace $V\subset L_2$, 
with 
$X=\{x_1,\dots,x_N\}$ being the sampling points, 
may lead with high probability to a near-optimal algorithm 
for approximation
(in the worst-case setting) if 
$x_1,\dots,x_N$ are randomly and independently chosen from the uniform distribution in $[0,1]^d$, 
and 
$N\sim \dim(V)\cdot\log(\dim(V))$, see~\cite{KU1,U2020}.

This follows from the following more general result. 
We refer to the survey~\cite{SU} where this, 
see~Sections~3.4 and~4.1,
as well as the randomized setting and generalizations to other classes $F$ and $G$ are explained.  

\begin{prop}\label{prop:ls}
Let $\mu$ be a Borel probability measure on a compact topological space $D$, 
$F\subset C(D)$ be compact, 
$\{b_1,b_2,\ldots\}$ be an orthonormal basis of $L_2(\mu)$ 
with $\sup_k\|b_k\|_\infty<\infty$, 
and $U_n:={\rm span}\{b_1,\dots,b_n\}$. 
Moreover, assume that 
\be\label{eq:cond-sampling}
\sup_{f \in F}\, \inf_{g\in U_n}\Vert f - g  \Vert_2 
\;\le\; K\cdot n^{-t} 
\ee
for some $t>1/2$, $K<\infty$ and all $n\in\N$. 

Then, there is a constant $C\in\N$ such that for all $n\in\N$, the least-squares method 
$\widehat{f}^{ls}_{N}:=\widehat{f}^{ls}_{U_n,X}$ from~\eqref{eq:alg-ls} 
with the $N \ge C\, n\log(n)$ points $X=\{x_1,\ldots,x_N\}$ being chosen iid w.r.t.~$\mu$ 
satisfies with probability $1-\frac{C}{N^2}$ 
and 
for every $2\le p\le\infty$ the bound 
\begin{equation*}
\sup_{f\in F}\, \Vert f - \widehat{f}^{ls}_{N} \Vert_p 
\;\le\; C\,K\, \left(\frac{N}{\log N}\right)^{-t+1/2-1/p}.
\end{equation*}
\end{prop}

\begin{proof}
Similarly to~\cite[Corollary~3]{KPUU23}, 
see also~\cite[Theorem~3.7]{SU}, 
we apply the algorithm from~\cite{KU1} together with~\cite[Lemma~10]{KPUU23}, 
and observe that the condition on $\{b_k\}$ allows to 
remove the weights in the algorithm, see~\cite[Remark~1]{KU1}, 
and that $\|g\|_p\le c n^{1/2-1/p}\|g\|_2$ 
for all $g\in U_n$ and $2\le p\le\infty$, see~\cite[Remark~5]{KPUU23}. \\
\end{proof}

\medskip

\begin{rem} \label{rem:subsampling}
For~\eqref{eq:alg-ls} to be uniquely solvable, it is required to have $N\ge \dim(V)$. 
The additional logarithmic factor in the results above is needed due to the assumption that the sampling points are chosen independently and identically distributed,  
see~\cite{KNS22}. 
This \emph{oversampling} can be removed with the help of 
(usually non-constructive) subsampling, 
see e.g.~\cite{DKU,KPUU23} for theoretical results 
and~\cite{BSU} for a detailed treatment of the implementation. 
Note that these approaches may require additional weights in the algorithm~\eqref{eq:alg-ls}. 
\end{rem}

\medskip

Our approach to bound the NN-sampling numbers is to apply Proposition \ref{prop:ls} to the Sobolev spaces $W^s([0,1]^d)$ and $V_R$ from~\eqref{eq:VR}.
Unfortunately, 
it is well known that in this case we can only have $t\le s/d$ in~\eqref{eq:cond-sampling}, 
even if we would choose the optimal subspaces there.
Since the analysis above requires $s<1$, and we need $t>1/2$, we do not get a result (yet) from this general approach for $d>1$. For $d=1$, however, we obtain the following. 

\begin{cor}
For $1/2<s<1$, there is a constant $C_s\in\N$ such that for all $R\in\N$, the least-squares method 
$\widehat{f}^{ls}_{N}:=\widehat{f}^{ls}_{V_R,X}$ from~\eqref{eq:alg-ls} with $V_R$ from~\eqref{eq:VR} 
and the $N \ge C\, R\log(R)$ points $X=\{x_1,\ldots,x_N\}$ being chosen iid w.r.t.~the uniform distribution on $[0,1]$ 
satisfies with probability $1-\frac{C}{N^2}$ 
and 
for every $2\le p\le\infty$ the bound 
\begin{equation*} 
\Vert f - \widehat{f}^{ls}_{N} \Vert_p 
\;\le\; C\, 
\left(\frac{N}{\log N}\right)^{-s+1/2-1/p}\; \|f\|_{W^s}
\end{equation*}
for all $f\in W^s([0,1])$ simultaneously. 
\end{cor}

We leave out the details of the proof, which relies on Proposition~\ref{prop:ls} 
but requires a few technicalities when using a Riesz basis instead of an orthonormal basis. 
Note that we could also take $\|f^R\|_{W^s}$ instead of $\|f\|_{W^s}$ above, see e.g.~\cite[Theorem~3.2]{SU}.

With this, Lemma~\ref{lem:ANN}, and taking Remark~\ref{rem:subsampling} into account, 
we obtain 
\begin{equation*}
g_n^{{\rm NN}}(W^s,L_p,bn,\log_2(bn)) 
\;\asymp\; n^{-s+1/2-1/p} 
\;\asymp\; c_n(W^s,L_p)   
\end{equation*}
for $W^s:=W^s([0,1])$ with $1/2<s<1$, 
$2\le p\le\infty$ 
and some absolute constant $b\in\N$. 
See e.g.~\cite[Theorem VII.1.1]{Pinkus85} for the classical second equivalence, and~\cite[Section~1.1]{KPUU23} for a similar discussion. 
That is, NN-sampling numbers behave as the Gelfand numbers, 
whenever we allow 
$W\asymp n$ and $L\asymp\log n$. 
(Alternatively, 
$W=4$ and $L\asymp n\log n$ would also work.) 
It would be of great interest to see whether and how 
this can be extended to larger $s$ and $d$.
However, note that unit balls of Sobolev spaces are, 
depending on the specific norm of the space, 
too large to get useful worst-case error bounds in high dimensions: 
Any algorithm needs at least 
$(d/\eps)^{d/k}$ 
function evaluations to achieve an error $\eps>0$, if the unit ball contains all functions with directional derivatives of order $k$ bounded by one, see~\cite{HNUW14,HNUW17}.

\bigskip

The Barron spaces ${\mathbb B}^s([0,1]^d)$ are more interesting in the present context, 
because they are much better suited for high dimensions. 
In particular, ${\mathbb B}^s([0,1]^d)$ is 
continuously embedded into $C([0,1]^d)$ for all $d\in\N$ and $s\ge0$.
However, for these classes, it is known that linear methods are not optimal. 
The method of choice is \emph{basis pursuit denoising}, i.e.,
\begin{equation} \label{eq:alg-bp}
\widehat{f}^{bp}_{R,X}
\,:=\, 
\underset{g\in V_R}{\rm argmin}\; 
\|g\|_{\mathcal{B}^0} 
\qquad\text{subject to}\quad
\sqrt{\frac1N\sum_{i=1}^N \vert g(x_i) - f(x_i) \vert^2} \,\le\, C\,R^{-s}, 
\end{equation}
with $C$ chosen as in~\eqref{eq:Riesz-Barron}, and $X=\{x_1,\dots,x_N\}$ are the sampling points. 
Recall that the $\mathcal{B}^0$-norm is just the absolute sum of the coefficients. 

This is the most important method of \emph{sparse approximation}  or \emph{compressed sensing}. 
We refer to~\cite{CT06,D06} for its origins and~\cite{FR13} for a detailed treatment of the subject. Also note that the idea already appeared implicitly in~\cite{GG84} in the context of $\ell_2$-approximation of $\ell_1$-vectors. 
The recovery from function values has been treated first in \cite{Rau07}, see also ~\cite{RW16}, 
and here we employ a variant of the method from~\cite{BDJR21} suitable for Riesz bases. 
This approach has been used for several problems in optimal recovery. Let us only highlight the recent contributions~\cite{JUV23,K24,Voigt1}, 
where this has been discussed in the context of sampling numbers, high dimensions and (different) Barron-type spaces, respectively. 
In our setting, we obtain the following from~\cite[Theorem~2.6]{BDJR21}. 

\begin{prop}\label{prop:bp}
For $0<s<1$, there is a constant $C_s\in\N$ such that for all $k\in\N$, the basis pursuit denoising  
$\widehat{f}^{bp}_{R,X}$
from~\eqref{eq:alg-bp} 
with the points $x_1,\ldots,x_N$ being chosen iid w.r.t.~the uniform distribution on $[0,1]^d$ and 
\begin{equation*}
    N \ge C_s\, k \,\log^2(k) \, \log\bigl(N(R,d)\bigr), 
\end{equation*}
with $N(R,d)$ from~\eqref{def:N(t,d)}, 
satisfies with probability $1-\frac{C_s}{k^2}$ 
and 
for every $2\le p\le\infty$ the bound 
\begin{equation}\label{eq:Lp-error-bp}
\Vert f - \widehat{f}^{bp}_{X,R} \Vert_p 
\;\le\; C_s \left( k^{-1/p}\, \sigma_k(f,\mathcal{B}^0,V_R) 
\,+\, k^{1/2-1/p}\, R^{-s}\,\|f\|_{\BB^s} \right)
\end{equation} 
for all $f\in \BB^s([0,1]^d)$ simultaneously. 
\end{prop}

\begin{proof}
This is a rather direct application of~\cite[Theorem~2.6]{BDJR21}, see also~\cite[Lemma~9]{K24}. 
For the precise form of~\eqref{eq:alg-bp}, 
it remains to observe that 
\[
\Vert f^R \Vert_\infty 
\,\le\, C\, R^{-s} \,\|f\|_{\BB^s}
\]
with $f$ as in~\eqref{eq:f_Riesz_d} and 
$f^R:=\sum_{\ell\,\righttriangleplus\,0: \|\ell\|_2>R} \Bigl[\alpha_\ell\CalC_\ell+\beta_\ell\CalS_\ell\Bigr]$. 
This follows the lines of~\eqref{eq:ANN-Bs-calc} using
$\|f\|_\infty\le \|f\|_{\CalB^0}$ for all $g\in\CalB^0=\BB^0$, 
since the Riesz system is bounded by one.
\end{proof}

\medskip

A careful analysis of $\sigma_k(f,\mathcal{B}^0,V_R)$ shall lead to bounds in~\eqref{eq:Lp-error-bp} that also reflect the smoothness $s$, possibly optimal and/or with a mild dependence on $d$. 
However, since we only consider small $s$ anyway, we only use the obvious 
$\sigma_k(f,\mathcal{B}^0,V_R)\le \|f\|_{\mathcal{B}^0} \le C'_s \|f\|_{\BB^s}$, 
where the latter inequality follows from Theorem~\ref{thm:Barron}, 
to obtain the following corollary. 
This shows that a linear-in-$d$ number of samples is enough in dimension $d$.

\begin{cor}
For $0<s<1$ and $2\le p\le\infty$, 
there is $C\in\N$, independent of $d$, 
such that the following holds. 
For $\eps>0$ and 
\begin{equation*}
    N \ge C\, d\, \eps^{-p} \,\log^3(1/\eps) 
\end{equation*}
the basis pursuit denoising  
$\widehat{f}^{bp}_{R,X}$
from~\eqref{eq:alg-bp} 
with $R=\eps^{-p/(2s)}$ and the points $x_1,\ldots,x_N$ being chosen iid w.r.t.~the uniform distribution on $[0,1]^d$
satisfies with probability $1-C \eps^2$ 
the bound 
\begin{equation*}
\Vert f - \widehat{f}^{bp}_{X,R} \Vert_p 
\;\le\; \eps\cdot\|f\|_{\BB^s}
\end{equation*} 
for all $f\in \BB^s([0,1]^d)$ simultaneously. 
\end{cor}

\begin{proof}
We apply Proposition~\ref{prop:bp} with 
$R=\eps^{-p/(2s)}$ and $k=\lceil(C_s/\eps)^{p} \rceil$, 
and use the (rough) bound 
$N(R,d)\le (3R)^d$ for $R\ge1$. 
\end{proof}

\bigskip

As already used in Theorem~\ref{thm:ANN_Bs}, the approximation 
$\widehat{f}^{bp}_{X,R}$ 
can be represented as an ANN from~$\Upsilon^{W,L}_d$ whenever 
 $W\ge4k\asymp (1/\eps)^{p}$ and 
$L\ge 4+\log_2(R\cdot \sqrt{\min(R,d)}) \asymp \log(1/\eps)$. 

Rephrasing this with the number of sample points, we obtain for 
$\BB^s=\BB^s([0,1]^d)$ with $0<s<1$ 
that 
\begin{equation*}
g_n^{{\rm NN}}(\BB^s,L_p,W,L) 
\;\le\; C\,\left(\frac{d}{n} \cdot \log^3(n/d)\right)^{1/p} 
\end{equation*}
for $W\ge C\,n/(d\,\log^3(n))$ and $L\ge C\,\log(n/d)$, where 
$C>0$ only depends on $p$ and $s$. 
(Again, we may use ANNs with $W=d+3$ and $L\ge C\,n/(d\log(n))$.)

This should be compared to~\cite[Theorem~3]{Barron2} where a slightly smaller linear-in-$d$ bound has been observed for ANNs with $L=1$ and $W\asymp \sqrt{n/(d\log(n))}$, but only for $p=2$ and a different Barron-type space (with $s=1$). 
Again, the advantage of our approach is that we can rely on techniques from linear approximation in order to treat deep networks  without using (very) tailored methods.



\thebibliography{99}


\bibitem{Barron3} A. R. Barron, Neural net approximation. In Proc. 7th Yale workshop on adaptive and learning systems, vol. 1 (1992), 69--72.

\bibitem{Barron} A. R. Barron. Universal approximation bounds for superpositions of a sigmoidal function,
\emph{IEEE Trans. Inf. Theory} {39}(3) (1993), 930--945.

\bibitem{Barron2} A. R. Barron. Approximation and estimation bounds for artificial neural networks.
\emph{Mach. Learn.} 14(1) (1994), 115--133.

\bibitem{BSU} F. Bartel, M. Sch\"afer, and T. Ullrich. Constructive subsampling of finite frames with applications in optimal function recovery,
\emph{Appl. Comput. Harmon. Anal.}  65 (2023), 209--248.

\bibitem{Bartlett} P. Bartlett, V. Maiorov, and R. Meir, Almost linear VC dimension bounds for piecewise polynomial networks,
Advances in neural information processing systems (NIPS) 11 (1998). 

\bibitem{BGKP} H. B\"olcskei, P. Grohs, G. Kutyniok, and P. Petersen.
Optimal approximation with sparsely connected deep neural networks, \emph{SIAM J. Math. Data Sci.} 1 (2019), 8--45.

\bibitem{BDJR21} S. Brugiapaglia and S. Dirksen and H. C. Jung and H. Rauhut. 
Sparse recovery in bounded Riesz systems with applications to numerical methods for PDEs, 
\emph{Appl. Comput. Harmon. Anal.} 53 (2021), 231--269.

\bibitem{CT06}
 E. J. Cand\`es and T. Tao. Near-optimal signal recovery from random projections: universal encoding strategies?,
 \emph{IEEE Trans. Inform. Theory} 52 (2006), 5406--5425.

\bibitem{CPV23} A. Caragea, P. Petersen, and F. Voigtlaender. Neural network approximation and estimation of classifiers with classification boundary in a Barron class,
\emph{Ann. Appl. Probab.} {33}(4) (2023), 3039--3079.

\bibitem{CS03}
O. Christensen and D. T. Stoeva. $p$-frames in separable Banach spaces. \emph{Adv. Comput. Math.} 18 (2003), 117--126. 

\bibitem{DD+22} I. Daubechies, R. DeVore, S. Foucart, B. Hanin, and G. Petrova.  
Nonlinear approximation and (deep) ReLU networks, 
\emph{Constr. Appr.} {55} (2022), 127--172.


\bibitem{DeVore} R. DeVore. Nonlinear approximation, \emph{Acta Numer.} {7} (1998), 51--150.

\bibitem{DeVore-2} R. DeVore, B. Hanin, and G. Petrova. Neural network approximation, \emph{Acta Numer.} {30} (2021), 327--444.

\bibitem{DHM} R. DeVore, R. Howard, and C. Micchelli. Optimal nonlinear
approximation, \emph{Manuscripta Math.} {63}(4) (1989), 469--478.

\bibitem{DKU} M.~Dolbeault, D.~Krieg, and M.~Ullrich. 
A sharp upper bound for sampling numbers in $L_2$,
\emph{Appl. Comput. Harmon. Anal.} {63} (2023), 113--134.


\bibitem{D06} D. L. Donoho, Compressed sensing, \emph{IEEE Trans. Inform. Theory} 52 (2006), 1289--1306. 


\bibitem{DN21}
D.~Dung and V.~K.~Nguyen, Deep ReLU neural networks in high-dimensional approximation, 
\emph{Neural Networks} 142 (2021), 619--635.

\bibitem{EMW22}
W.~E, C.~Ma, and L.~Wu. The Barron space and the flow-induced function spaces for neural network models, \emph{Constr Appr.} 55 (2022), 369--406. 

\bibitem{EW20}
W.~E and  S.~Wojtowytsch.
A priori estimates for classification problems using neural networks (2020), 
arXiv: 2009.13500.

\bibitem{EW-Banach}
W.~E and S.~Wojtowytsch. 
On the Banach Spaces Associated with Multi-Layer ReLU Networks: Function Representation, Approximation Theory and Gradient Descent Dynamics, \emph{CSIAM Transactions on Applied Mathematics} {1(3)} (2020), 387--440. 


\bibitem{E1} W. E and S. Wojtowytsch. Representation formulas and pointwise properties for Barron functions,
\emph{Calc. Var. Partial Differ. Equ.} 61(2) (2022), 1--37.

\bibitem{EPGB}
D. Elbr\"achter, D. Perekrestenko, P. Grohs, and H. B\"olcskei.
Deep Neural Network Approximation Theory,
\emph{IEEE Trans. Inf. Theory} 67(5) (2021), 2581--2623.

\bibitem{FR13} 
S. Foucart and H. Rauhut. 
\emph{A Mathematical Introduction to Compressive Sensing},
Birkh\"auser, New York (2013).

\bibitem{GG84}
A. Yu. Garnaev and E. D. Gluskin. The widths of a Euclidean ball, \emph{Soviet Math. Dokl.} 30 (1984), 200--204.


\bibitem{GKP} I. G\"uhring, G. Kutyniok, and P. Petersen. Error bounds for approximations with deep ReLU neural networks in $W^{s,p}$ norms, \emph{Anal. Appl.} 18.05 (2020), 803--859.

\bibitem{HLP} G. H. Hardy, J. E. Littlewood, and G. P\'olya. \emph{Inequalities}, Cambridge University Press (1936). 


\bibitem{HLS} H. Hedenmalm, P. Lindqvist and K. Seip. A Hilbert space of Dirichlet series
and systems of dilated functions in $L_2(0,1)$, \emph{Duke Math. J.} 86 (1997), 1--37.

\bibitem{HNUW14}
A. Hinrichs, E. Novak, M. Ullrich, H. Wo\'zniakowski,
The curse of dimensionality for numerical integration
of smooth functions, \emph{Math. Comp.} 83 (2014), 2853--2863.

\bibitem{HNUW17}
A. Hinrichs, E. Novak, M. Ullrich, H. Wo\'zniakowski, 
Product rules are optimal for numerical integration in classical 
smoothness spaces, 
\emph{J. Complexity} 38 (2017), 39--49.

\bibitem{JUV23}
 T. Jahn, T. Ullrich, F. Voigtlaender. Sampling numbers of smoothness classes
via $\ell_1$-minimization, \emph{J. Compl.} 79 (2023), 101786.


\bibitem{Jones} L. K. Jones. A simple lemma on greedy approximation in Hilbert space and convergence rates for projection
pursuit regression and neural network training, \emph{Ann. Stat.} 20(1) (1992), 608--613.



\bibitem{KB} J. M. Klusowski and A. R. Barron. Approximation by combinations of $\ReLU$ and squared $\ReLU$ ridge functions with $\ell^1$ and $\ell^0$ controls,
\emph{IEEE Trans. Inform. Theory} 64(12) (2018), 7649–7656.

\bibitem{K24}
D. Krieg. Tractability of sampling recovery on unweighted function classes, 
\emph{Proc. Amer. Math. Soc. Ser. B} 11 (2024), 115--125.


\bibitem{KNS22} D.~Krieg, E.~Novak, and M.~Sonnleitner. 
{Recovery of Sobolev functions restricted to iid~sampling}, 
\emph{Math. Comp.} {91} (2022), 2715--2738.

\bibitem{KPUU23}
D.\,Krieg, K.\,Pozharska, M.\,Ullrich, and T.\,Ullrich.
Sampling recovery in $L_2$ and other norms,
to appear in \emph{Math. Comp.},  arXiv:2305.07539.


\bibitem{KU1} D. Krieg and M. Ullrich. 
Function values are enough for $L_2$-approximation, 
\emph{Found.~Comput.~Math.} {21} (2021), 1141--1151.

\bibitem{KMU} T. K\"uhn, S. Mayer, and T. Ullrich.  
Counting via entropy: new preasymptotics for the approximation numbers of Sobolev embeddings,
\emph{SIAM J. Numer. Anal.} {54}(6) (2016), 3625--3647.


\bibitem{Per24} Y. Li, S. Lu, P. Math\'e, and S. Pereverzyev. Two-layer networks with the $\mathrm{ReLU}^k$ activation function: Barron spaces and derivative approximation,  \emph{Numer. Math.} {156} (2024), 319--344. 


\bibitem{LK} P. Lindqvist and K. Seip. Note on some greatest common divisor matrices, \emph{Acta Arith.} {84 (2)} (1998), 149--154.


 
\bibitem{40} J. Lu, Z. Shen, H. Yang, and S. Zhang. Deep network approximation for smooth functions, \emph{SIAM J. Math. Anal.} {53(5)} (2021), 5465--5506.

\bibitem{Mhaskar'} H. N. Mhaskar, Approximation properties of a multilayered feedforward artificial neural network,
\emph{Adv. Comput. Math.} 1 (1993), 61--80.

\bibitem{Mhaskar} H. N. Mhaskar, Neural networks for optimal approximation of smooth and analytic functions,
\emph{Neural Comput.} 8(1) (1996),  164--177.

\bibitem{MLP16} H. N. Mhaskar, Q. Liao, and T. Poggio,  Learning functions: When is deep better than shallow (2016), arXiv: 1603.00988. 


\bibitem{MP16} H. N. Mhaskar and T. Poggio, Deep vs. shallow networks: An approximation theory perspective, 
\emph{Anal. Appl.}  14(6) (2016), 829--848. 




\bibitem{Petersen-book} P. Petersen and J. Zech. Mathematical theory of deep learning (2024), arXiv:2407.18384.


\bibitem{Pinkus85}
A. Pinkus. {\it $n$-widths in approximation theory}, 
Ergebnisse der Mathematik und ihrer Grenzgebiete. 3. Folge / A Series of Modern Surveys in Mathematics (MATHE3, volume 7), Springer Berlin, Heidelberg (1985). 


\bibitem{Pinkus99}
A. Pinkus. Approximation theory of the MLP model in neural networks, \emph{Acta Numerica} 8 (1999), 143--195. 



\bibitem{Pisier} G. Pisier, Remarques sur un r\'esultat non publi\'e de B. Maurey (Remarks on an unpublished result of B. Maurey),
\'Ecole Polytechnique, Centre de Math\'ematiques, Palaiseau (1981). 

\bibitem{Ragu} M. Raghu, B. Poole, J. Kleinberg, S. Ganguli, and J. Sohl-Dickstein, On the expressive power of deep neural networks, In international conference on machine learning,
\emph{Proc. Mach. Learn. Res.} (2017), 2847--2854.


\bibitem{Rau07} 
H. Rauhut. Random sampling of sparse trigonometric polynomials, \emph{Appl. Comput. Harmon. Anal.} 22 (2007), 16--42.


\bibitem{RW16} H. Rauhut and R.Ward, Interpolation via weighted $\ell_1$-minimization, \emph{Appl. Comput. Harmon. Anal.} 40  (2016),  321--351.

\bibitem{Schmidt} E. Schmidt. Zur Theorie der linearen und nichtlinearen Integralgleichungen I, \emph{Math. Ann.} {63} (1907), 433--476.

\bibitem{SV} C. Schneider and J. Vyb\'\i ral. A multivariate Riesz basis of ReLU neural networks, {\em Appl. Comput. Harmon. Anal.} {68} (2024), 101605. 

\bibitem{50} Z. Shen, H. Yang, and S. Zhang. Optimal approximation rate of ReLU networks in terms of width and depth,
\emph{J. Math. Pures Appl.} {157} (2022), 101--135.

\bibitem{S23} J. W. Siegel. {Optimal approximation rates for deep ReLU neural networks on Sobolev and Besov spaces}, \emph{J. Mach. Learn. Res.} {24}  (2023), 1--52.

\bibitem{Siegel21} J. W. Siegel and J. Xu. High-order approximation rates for shallow neural networks with cosine and $\mathrm{ReLU}^k$ activation functions, {\em Appl. Comput. Harmon. Anal.} {58} (2021), 1--26. 

\bibitem{SXu} J. W. Siegel and J. Xu, Sharp bounds on the approximation rates, metric entropy, and $n$-widths of shallow neural networks   (2021), arXiv: 2101.12365.

\bibitem{SU}
M.~Sonnleitner and M.~Ullrich, On the power of iid information for linear approximation, 
\emph{J. Appl. Numer. Anal.} 1 (2023), 88--126.

\bibitem{T} M. Telgarsky. Representation benefits of deep feedforward networks, {arXiv:1509.08101} (2015).

\bibitem{T2} M. Telgarsky. Benefits of depth in neural networks, Conference on learning theory, 29th Annual Conference on Learning Theory, \emph{Proc. Mach. Learn. Res.} 49 (2016), 1517--1539.

\bibitem{Tit} E. C. Titchmarsh. \emph{The theory of the Riemann zeta-function}, The Clarendon Press, Oxford University Press, New York (1986).

\bibitem{U2020} M. Ullrich. On the worst-case error of least squares algorithms for $L_2$-approximation with high probability, \emph{J.~Complex.} 60 (2020), 101484.

\bibitem{Versh} R. Vershynin, \emph{High-dimensional probability: An introduction with applications in data science}, volume 47, Cambridge university press (2018).

\bibitem{Voigt1} F. Voigtlaender. $L_p$-sampling numbers for the Fourier-analytic Barron space, {arXiv:2208.07605} (2022).

\bibitem{Y17} D. Yarotsky. {Error bounds for approximations with deep ReLU networks}, \emph{Neural Networks} {94} (2017), 103--114.

\bibitem{61} D. Yarotsky. Optimal approximation of continuous functions by very deep ReLU networks, {\em Conference on Learning Theory},
\emph{Proc. Mach. Learn. Res.} (2018), 639--649.

\bibitem{W} A. Wintner. Diophantine approximations and Hilbert’s space, \emph{Amer. J. Math.} {66} (1944), 564--578.
\end{document}